\documentclass{tlp}

\usepackage{caption}
\usepackage{subcaption}
\usepackage{graphicx}
\usepackage{color}
\usepackage[fleqn]{amsmath}
\usepackage{amssymb}
\usepackage[all]{xy}
\usepackage{float}
\usepackage{wrapfig}
\usepackage{xspace}
\usepackage{url}
\usepackage{hyperref}
\usepackage[compact]{titlesec}
\usepackage{pdfpages}

\floatstyle{boxed}
\restylefloat{figure}

\newfloat{program}{thp}{}
\floatname{program}{Program}

\newcommand{\comment}[1]{{}}

\newcommand{\myc}[1]{}

\newtheorem{Lem}{Lemma}

\newtheorem{Thm}[Lem]{Theorem}
\newtheorem{Pro}[Lem]{Proposition}

\newtheorem{Prop}[Lem]{Property}
\newtheorem{Def}{Definition}

\newtheorem{Ex}{Example}

\newcommand{\id}[1]{\mbox{\it #1\/}}

\newcommand{\kw}[1]{\mbox{\tt #1}}

\def\theenumiii{\roman{enumiii}}
\def\p@enumiii{\theenumi(\theenumii)}

\renewcommand{\subparagraph}[1]{\smallskip \noindent \textbf{\textit{#1}} \hspace*{0.5em}}
\newcommand{\mgu}{\mbox{\rm mgu\xspace}}
\newcommand{\rv}{\mbox{\rm rv\xspace}}

\def\squareforqed{\hbox{\rlap{$\sqcap$}$\sqcup$}}
\def\qed{\ifmmode\squareforqed\else{\unskip\nobreak\hfil
\penalty50\hskip1em\null\nobreak\hfil\squareforqed
\parfillskip=0pt\finalhyphendemerits=0\endgraf}\fi}

\newcommand{\cpdf}[2]{\ensuremath{\langle#1,\mbox{\textcolor{blue}{$#2$}}\rangle}}

\title{Inference in Probabilistic~Logic~Programs with Continuous~Random~Variables}

\author[M.A.Islam, C.R. Ramakrishnan, I.V. Ramakrishnan]{ {\bf Muhammad Asiful Islam, C.R. Ramakrishnan, I.V. Ramakrishnan} \\
Dept. of Computer Science\\
Stony Brook University\\
Stony Brook, NY 11794 \\
\{maislam, cram, ram\}@cs.sunysb.edu\\
}

\begin{document}

\maketitle

\begin{abstract}
\begin{quote}
  Probabilistic Logic Programming (PLP), exemplified by Sato and
  Kameya's PRISM, Poole's ICL, Raedt et al's ProbLog and Vennekens et
  al's LPAD, is aimed at combining statistical and logical knowledge
  representation and inference.  However, the
  inference techniques used in these works rely on enumerating sets of
  explanations for a query answer.  Consequently, these languages
  permit very limited use of random variables with continuous
  distributions.  In this paper, we present a \emph{symbolic}
  inference procedure that uses constraints and represents sets of
  explanations without enumeration.  This permits us to reason over
  PLPs with Gaussian or Gamma-distributed random variables (in
  addition to discrete-valued random variables) and linear equality
  constraints over reals.  We develop the inference procedure in the
  context of PRISM; however the procedure's core ideas can be easily
  applied to other PLP languages as well.  An interesting aspect of
  our inference procedure is that PRISM's query evaluation process
  becomes a special case in the absence of any continuous random
  variables in the program.  The symbolic inference procedure enables
  us to reason over complex probabilistic models such as Kalman
  filters and a large subclass of Hybrid Bayesian networks that were
  hitherto not possible in PLP frameworks.
  (To appear in Theory and Practice of Logic Programming)
\end{quote}
\end{abstract}


\section{Introduction}
\label{sec:intro}

Logic Programming (LP) is a well-established language model for
knowledge representation based on first-order logic.  Probabilistic
Logic Programming (PLP) is a class of Statistical Relational Learning
(SRL) frameworks~\cite{srlbook} which are designed for combining
statistical and logical knowledge representation.

The semantics of PLP languages is defined based on the semantics
of the underlying non-probabilistic logic programs.  A large class of
PLP languages, including ICL~\cite{PooleICL},
PRISM~\cite{sato-kameya-prism}, ProbLog~\cite{deRaedt} and
LPAD~\cite{lpad}, have a declarative \emph{distribution semantics},
which defines a probability distribution over possible models of the
program.  Operationally, the combined statistical/logical inference is
performed based on the proof structures analogous to those created by
purely logical inference.  In particular, inference proceeds as in
traditional LPs except when a random variable's valuation is used.
Use of a random variable creates a branch in the proof structure, one
branch for each valuation of the variable. Each proof for an answer is
associated with a probability based on the random variables used in
the proof and their distributions; an answer's probability is
determined by the probability that at least one proof holds.  Since
the inference is based on enumerating the proofs/explanations for
answers, these languages have limited support for continuous random
variables.  We address this problem in this paper. 
A comparison of our work with recent efforts at extending other
SRL frameworks
to continuous variables appears in Section \ref{sec:related}.

We provide an inference procedure to reason over PLPs with Gaussian or
Gamma-distributed random variables (in addition to discrete-valued
ones), and linear equality constraints over values of these continuous
random variables.  We describe the inference procedure based on
extending PRISM with continuous random variables.  This choice is
based on the following reasons.  First of all, the use of explicit
random variables in PRISM simplifies the technical development.
Secondly, standard statistical models such as Hidden Markov Models
(HMMs), Bayesian Networks and Probabilistic Context-Free Grammars
(PCFGs) can be naturally encoded in PRISM.  Along the same lines, our
extension permits natural encodings of Finite Mixture Models (FMMs)
and Kalman Filters.  Thirdly, PRISM's inference
naturally reduces to the Viterbi
algorithm~\cite{Viterbi} over HMMs, and the Inside-Outside
algorithm~\cite{InsideOutside} over PCFGs.  \comment{This enables us to derive
an inference procedure that naturally reduces to the ones used for
evaluating Finite Mixture Models and Kalman Filters.}  The
combination of well-defined model theory and efficient inference has enabled the use of PRISM for
synthesizing knowledge in sensor networks~\cite{Sensys}.

It should be noted that, while the technical development in this paper is
limited to PRISM, the basic technique itself is applicable to other
similar PLP languages such as ProbLog and LPAD (see Section \ref{sec:extn}).

\comment{

Specifically, a logic program consists of a data base of facts
representing 
background domain knowledge and
 logical relationships (encoded as rules) that describe the relationships which hold amongst entities
 in the application domain. The user poses
queries and the run time system searches through the facts and rules to  logically deduce  answers to the queries using a theorem-proving approach  \cite{Lloyd}. Prolog is a an exemplar of a LP language \cite{Bratko2000}.

Traditionally   query evaluation in LP systems  deduces  unambiguous truths/falsehoods  of the relationships expressed in the program and hence   is not geared to deal with uncertainty in the relationships.  Statistical reasoning on the other hand naturally deals with uncertainty.
Thus combining statistical and logical reasoning in LP systems  has been and continues to be an important and fertile
research topic   (e.g. see
\cite{NgSub,Poole,lakshman,Muggleton,srlbook,PooleICL}).

One notable work on this topic is that of  Sato and Kameya \cite{sato-kameya-prism} who
developed the PRISM  language,  an extension of LP to include probabilities.
PRISM, unlike prior attempts at incorporating probabilities  in a LP
framework, allows the  use of random variables explicitly in the
program.  It combines logical and statistical modeling by allowing one to use
outcomes of trials of random variables in a logic program.
(a detailed overview of PRISM appears in a later section). 
 Standard
statistical models such as Hidden Markov Models (HMMs), Bayesian
Networks and Probabilistic Context-Free Grammars (PCFGs) can be
readily encoded in PRISM.  A fundamental contribution of PRISM is its
model theory, called \emph{distribution semantics}; analogous to the
least Herbrand model semantics for a logic program, a PRISM program
has a distribution of least models.  The parameters of the
distribution of models is obtained from the parameters of the random
variables used in the program.  PRISM also has a proof
theory~\cite{sato} that closely follows OLDT
resolution~\cite{OLDT} and gives an efficient inference procedure for
a large class of logical-statistical models.  In fact, PRISM's inference
naturally reduces to the Viterbi algorithm~\cite{Viterbi} over HMMs, and the
Inside-Outside algorithm~\cite{InsideOutside} over PCFGs.  This
combination of well-defined
model theory and efficient inference has enabled the use of PRISM for
synthesizing knowledge in sensor networks~\cite{Sensys}.

A major shortcoming of  PRISM is that it permits only finite-domain (i.e. discrete-valued)  random variables. This precludes
representing or evaluating a large class  of
probabilistic models that use continuous-domain
random variables.  We address this problem in this paper.  A
comparison of our work with recent efforts at extending other
statistical relational frameworks to
continuous variables appears in Section~\ref{sec:related}.
}

\comment{There has been a very recent attempt at addressing this problem in \cite{hproblog}. While a more detailed comparison appears in Section~\ref{sec:related}  on Related Works suffice it is to say here that this work cannot handle interesting problems such as Kalman Filters \cite{aibook} and Hybrid Bayesian Networks where the parent/child conditional dependencies are captured using arbitrary discrete/continuous random variable combinations \cite{hbn}. Thus a rigorous and extensive  treatment of probabilistic LP along the lines of PRISM semantics,  that includes both discrete and continuous random variables remains open and is the topic addressed in this paper.}

\paragraph{Our Contribution:}
We extend  PRISM at the language level  to seamlessly include  discrete as well as continuous random variables.  We develop a new inference  procedure to evaluate queries over such extended PRISM programs.
\begin{itemize}
\item  We extend the PRISM language for specifying distributions of
  continuous random variables, and linear equality constraints over such variables.
\item We develop a
  \emph{symbolic inference}  technique to reason with constraints on
  the random variables. PRISM's inference technique becomes a special
  case of our technique when restricted to logic
  programs with discrete random variables.
\item These two developments enable the encoding of rich statistical models such as
  Kalman Filters and a large class of Hybrid Bayesian Networks; and
  exact inference over such models, which
  were hitherto not  possible  in LP  and its probabilistic extensions. 
\end{itemize}

Note that the technique of using PRISM for in-network evaluation of
queries in a sensor network~\cite{Sensys} can now be applied directly when sensor
data and noise are continuously distributed.  Tracking and navigation
problems in sensor networks are special cases of the Kalman Filter
problem~\cite{DSN}.  There are a number of other network inference
problems, such as the indoor localization problem, that have been
modeled as FMMs~\cite{WiGEM}.  Moreover, our
extension permits reasoning over models with finite mixture
of Gaussians and discrete distributions (see Section~\ref{sec:extn}).  Our extension of PRISM
brings us closer to the ideal of finding a declarative basis for
programming in the presense of noisy data.

\comment{
The primary focus of this paper is the development of inference
techniques for probabilistic logic programs.  The important topic of
learning the probability distributions of the random variables in
extended PRISM programs from facts is a topic of future work.  
}The
rest of this paper is organized as follows.  We begin with a review of
related work in Section~\ref{sec:related}, 
and describe the PRISM framework in detail in Section~\ref{sec:prism}.  We
introduce the extended PRISM language and the symbolic inference
technique for the extended language in Section~\ref{sec:inference}.  
In section~\ref{sec:kalman} we show the use of this
technique on an example encoding of the Kalman Filter.  We conclude in Section~\ref{sec:extn} with a discussion on extensions to our inference procedure.


\section{Related Work}
\label{sec:related}

Over the past decade, a number of Statistical Relational Learning
(SRL) frameworks have been developed, which support modeling,
inference and/or learning using a combination of logical and
statistical methods.  These frameworks can be broadly classified as
statistical-model-based or logic-based, depending on how their
semantics is defined.  In the first category are frameworks such as
Bayesian Logic Programs (BLPs)~\cite{blp}, Probabilistic Relational
Models (PRMs)~\cite{prm}, and Markov Logic Networks (MLNs)~\cite{mln},
where logical relations are used to specify a model compactly.  A BLP
consists of a set of Bayesian clauses (constructed from Bayesian
network structure), and a set of conditional probabilities
(constructed from CPTs of Bayesian network). PRMs encodes discrete
Bayesian Networks with Relational Models/Schemas.  An MLN is a set of
formulas in first order logic associated with
weights.  The semantics of a model in these frameworks is given in
terms of an underlying statistical model obtained by expanding the
relations. 
\comment{
For instance, a Markov Random Field (MRF) is constructed
from an MLN with nodes drawn from the set of all ground instances of
predicates in the first order formulas.  The set of (ground)
predicates in a ground instance of a formula forms a factor, and the
factor potential is obtained from the formula's weight and its truth
value.  Inference in an MLN is thus reduced to inference over the
underlying MRF. 
}

\sloppypar
Inference in SRL frameworks such as PRISM~\cite{sato-kameya-prism},
Stochastic Logic Programs (SLP)~\cite{Muggleton}, Independent Choice
Logic (ICL)~\cite{PooleICL}, and ProbLog~\cite{deRaedt} is primarily driven by
query evaluation over logic programs.  In SLP, clauses of a logic
program are annotated with probabilities, which are then
used to associate probabilities with proofs (derivations in a logic
program).  ICL~\cite{Poole} consists of definite clauses and disjoint
declarations of the form $disjoint([h_{1}:p_{1},...,h_{n}:p_{n}])$
that specifies a probability distribution over the hypotheses (i.e.,
$\{h_{1},..,h_{n}\}$).  Any probabilistic knowledge representable in a
discrete Bayesian network can be represented in this framework.  While
the language model itself is restricted (e.g., ICL permits only
acyclic clauses), it had declarative distribution semantics.  This
semantic foundation was later used in other frameworks such as PRISM
and ProbLog.  
CP-Logic~\cite{CPlogic} is a logical language to represent probabilistic
causal laws, and its semantics is equivalent to probability distribution over
well-founded models of certain logic programs.  Specifications in
LPAD~\cite{lpad} resemble those in CP-Logic: probabilistic predicates
are specified with disjunctive clauses, i.e. clauses with multiple
disjunctive consequents, with a distribution defined over the
consequents.  LPAD has a distribution semantics, and a proof-based
operational semantics similar to that of PRISM.
ProbLog
specifications annotate facts in a logic
program with probabilities.  In contrast to SLP, ProbLog has a
distribution semantics and a proof-based operational semantics.  PRISM
(discussed in detail in the next section), LPAD and ProbLog are
equally expressive.  PRISM uses explicit random variables and a simple
inference but restricted procedure.  In particular, PRISM demands that
the set of proofs for an answer are pairwise mutually exclusive, and
that the set of random variables used in a single proof are pairwise
independent.  The inference procedures of LPAD and ProbLog lift these
restrictions.  

SRL frameworks that are based primarily on statistical inference, such
as BLP, PRM and MLN, were originally defined over discrete-valued
random variables, and have been naturally extended to support a
combination of discrete and continuous variables.  Continuous
BLP~\cite{hblp} and Hybrid PRM~\cite{hprm} extend their
base models by using Hybrid Bayesian Networks~\cite{hbn}.  Hybrid
MLN~\cite{hmln} allows description of continuous properties and
attributes (e.g.,  the formula $\id{length}(x) = 5$ with weight $w$)
deriving MRFs with continuous-valued nodes (e.g., $\id{length(a)}$ for
a grounding of $x$, with mean $5$ and standard deviation
$1/\sqrt{2w}$).

In contrast to BLP, PRM and MLN, SRL frameworks that are primarily
based on logical inference offer limited support for continuous
variables.  In fact, among such frameworks, only ProbLog has been
recently extended with continuous variables.  Hybrid ProbLog
~\cite{hproblog} extends Problog by adding a set of continuous
probabilistic facts (e.g., $(X_{i},\phi_{i})::f_{i}$, where $X_{i}$ is
a variable appearing in atom $f_{i}$, and $\phi_{i}$ denotes its
Gaussian density function). It adds three predicates namely
\emph{below, above, ininterval} to the background knowledge to
process values of continuous facts.  A ProbLog program may use a
continuous random variable, but further processing
can be based only on testing whether or not the variable's value lies
in a given interval.  As a consequence, statistical models such as
Finite Mixture Models can be encoded in Hybrid ProbLog, but others
such as certain classes of Hybrid Bayesian Networks (with continuous
child with continuous parents) and Kalman Filters cannot be encoded.
The extension to PRISM described in this paper makes the framework
general enough to encode such statistical models.

More recently, \cite{apprProblog} introduced a sampling based approach for
(approximate) probabilistic inference in a ProbLog-like language that
combines continuous and discrete random variables. The inference
algorithm uses forward chaining and rejection sampling.  The language
permits a large class of models where discrete and continuous variables may be combined
without restriction.  
In contrast, we propose an exact inference algorithm with a more restrictive
language, but ensure that inference 
matches the complexity of specialized
inference algorithms for important classes of statistical models
(e.g., Kalman filters).


\section{Background: an overview of PRISM}
\label{sec:prism}

\comment{
\begin{figure}
  \small
  \begin{center}
    \begin{minipage}[t]{3in}
\begin{verbatim}
hmm(N, T) :-
    msw(init, S),
    hmm_part(0, N, S, T).

hmm_part(I, N, S, T) :-
    I < N, NextI is I+1,
    msw(trans(S), I, NextS),
    obs(NextI, A),
    msw(emit(NextS), NextI, A),
    hmm_part(NextI, N, NextS, T).
hmm_part(I, N, S, T) :- I=N, S=T.
\end{verbatim}
    \end{minipage}
  \end{center}
  \caption{PRISM program for an HMM}
  \label{fig:hmm-prism}
\end{figure}
}

PRISM programs have Prolog-like syntax (see Fig.~\ref{fig:hmm-prism}).
In a PRISM program the \texttt{msw} relation (``multi-valued switch'')
has a special meaning: \texttt{msw(X,I,V)} says that 
\texttt{V} is the outcome of the
\texttt{I}-th instance from a family \texttt{X} of random
processes\footnote{Following PRISM, we often omit the instance number
  in an \texttt{msw} when a program uses only one instance from a
  family of random processes.}.  The set of variables $\{\mathtt{V}_i
\mid \mathtt{msw(}p, i, \mathtt{V}_i\mathtt{)}\}$ are i.i.d. for a given random process $p$.
\comment{ The
\texttt{msw} relation provides the mechanism for using random
variables, thereby allowing us to weave together statistical and
logical aspects of a model into a single program.  }
The distribution
parameters of the random variables are specified separately.

The program in Fig.~\ref{fig:hmm-prism} encodes a Hidden Markov Model
(HMM) in PRISM. 
\begin{wrapfigure}{r}{0.48\textwidth}
\vspace{-5pt}
  \begin{center}
 \includegraphics[width=0.39\textwidth]{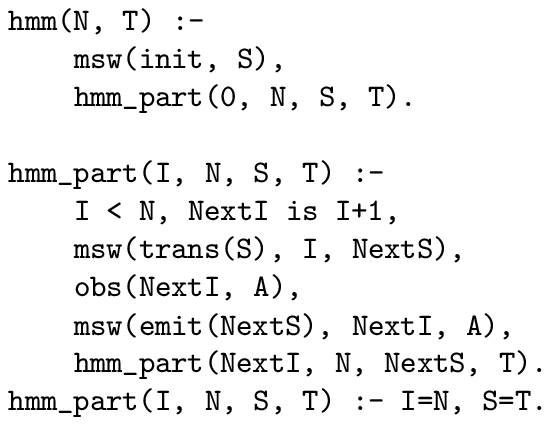}
 \end{center}
  \caption{PRISM program for an HMM}
  \label{fig:hmm-prism}
\vspace{-5pt}
\end{wrapfigure}
The set of observations is encoded as facts of predicate
\texttt{obs}, where \texttt{obs(I,V)} means that value \texttt{V} was
observed at time \texttt{I}.
In the figure, the clause defining
\texttt{hmm} says that \texttt{T} is the \texttt{N}-th state if we
traverse the HMM starting at an initial state \texttt{S} (itself the
outcome of the random process \texttt{init}).  In
\texttt{hmm\_part(I, N, S, T)}, \texttt{S} is the
\texttt{I}-th state, \texttt{T} is the \texttt{N}-th state. The first clause of
\texttt{hmm\_part} defines the conditions under which we go from
the \texttt{I}-th state \texttt{S} to the \texttt{I+1}-th state \texttt{NextS}.
\comment{
\begin{enumerate}
\item \texttt{msw(trans(S), I, NextS)} means that \texttt{NextS} is
  a random variable whose distribution depends on the value of \texttt{S};

\item \texttt{obs(NextI,A)} means that symbol \texttt{A} is at the
  \texttt{I+1}-th position in the observation sequence; and 

\item \texttt{msw(emit(NextS), NextI, A)} means that the observed symbol
  \texttt{A} is a random variable whose distribution depends on \texttt{NextS}.
\end{enumerate}
}
Random processes \texttt{trans(S)} and \texttt{emit(S)} give the distributions of
transitions and emissions, respectively, from state \texttt{S}.

The meaning of a PRISM program is given in terms of a
\emph{distribution semantics}~\cite{sato-kameya-prism,sato}.
A PRISM program is treated as a 
non-probabilistic logic program over a set of probabilistic facts, 
the \texttt{msw} relation.  An instance
of the \texttt{msw} relation defines one choice of values of all
random variables.  \comment{ A PRISM program, given an
instance of the \texttt{msw} relation, is a
non-probabilistic logic program, and its semantics is its
least Herbrand model.  Thus a}
A PRISM program is 
associated with a set of least models,
one for each \texttt{msw} relation instance.  A probability
distribution is then defined over the set of models, based on the
probability distribution of the \texttt{msw} relation instances.  This
distribution is the semantics of a PRISM
program.  Note that the distribution semantics is declarative.
For a subclass of 
programs, PRISM has an efficient procedure for
computing this semantics based on OLDT resolution~\cite{OLDT}.
\comment{, a
proof technique with memoization for definite logic programs.
}

\comment{
PRISM's inference can be understood, at a high level, in terms
of query evaluation in (non-probabilistic) logic programs
 which proceeds as
follows.  A query is, in general, a \emph{conjunction} of subgoals
$G_1, G_2, \ldots, G_n$, where each subgoal $G_i$ is an instance of a
predicate.  One step of evaluation is done by \emph{deriving} a new
query using a resolution rule, and applying such single-step
transformations until no further derivation is possible.
The sequence of queries generated by
this process is called a \emph{derivation}; a derivation is successful
if the sequence terminates in $\id{true}$.
}

\comment{
For instance, in OLD resolution used by Prolog, the
first subgoal is selected. Let  $(H :- B_1, \ldots B_k)$ be a clause
and let $\theta$ be the most general unifier (mgu) of $H$ and $G_1$:
i.e. $\theta$ is a most general substitution of variables
in $H$ and $G_1$ such that $H\theta = G_1\theta$.  If such a $\theta$
exists, then $B_1, \ldots, B_k, G_2, \ldots G_n$ is the next query in
the derivation.  Let $G$ be a subgoal with a successful derivation,
and let $\hat{\theta}$ be the composition of
mgu's used in each step in the derivation.  Then $G\hat{\theta}$ is an
\emph{answer} to $G$.  OLDT resolution maintains subgoals and answers
encountered during query evaluation
in an auxiliary table.  It uses an additional resolution
rule where a subgoal may be resolved with answers previously computed
for that subgoal.  Sharing results of computations using memo tables
makes OLDT efficient: enabling, e.g., cubic-time parsing for arbitrary
context-free grammars expressed as logic programs.
}

Inference in PRISM proceeds as follows.
When the goal selected at a step is of the form
\texttt{msw(X,I,Y)}, then \texttt{Y} is bound to a possible outcome of a
random process \texttt{X}.  
\emph{Thus in PRISM, derivations are constructed by enumerating the
  possible outcomes of each random variable.}
The derivation step is
associated with the probability of this outcome.  If all random
processes encountered in a derivation are independent, then the
probability of the derivation is the product of probabilities of each
step in the derivation.  If a set of derivations are pairwise mutually
exclusive, the probability of the set is the sum of probabilities of
each derivation in the set.   PRISM's evaluation procedure is
defined only when the independence and exclusiveness assumptions
hold.  
Finally, the probability of an answer 
is the probability of the set of derivations of that answer. 
\comment{The total probability of all answers to
a subgoal is called the \emph{inside} probability of the subgoal.}

\comment{
As an illustration, consider the query \texttt{hmm(n,T)} where
\texttt{n} is a fixed integer, evaluated over program in
Fig.~\ref{fig:hmm-prism}.   One step of resolution derives goal of
the form \texttt{msw(init,S), hmm\_part(0,n,S,T)}.  Now note that
there are a number of possible next steps: one for each value in the
range of \texttt{init}.  For instance, if the range of \texttt{init}
is $\{\mathtt{s0}, \mathtt{s1}\}$, there are two possible next steps:
\texttt{hmm\_part(0,n,s0,T)} and \texttt{hmm\_part(0,n,s1,T)}.
}

\comment{
Note that evaluation of \texttt{hmm(n,T)} over the program in
Fig.~\ref{fig:hmm-prism} obeys the independence and exclusiveness
assumptions of PRISM.  For instance, process \texttt{trans(S)}
at the \texttt{I}-th step is independent of process
\texttt{emit(NextS)} at the \texttt{I}+1-th step.  The only branches
in the proofs are due to different outcomes of same random process
(e.g., \texttt{init} described in the previous paragraph).
}

\comment{
Answers to \texttt{hmm(n,T)} and their probabilities give the
\emph{filter} distribution of the \texttt{n}-th state.  In addition to
answer probabilities, PRISM can compute the probability that a subgoal
$G'$ is encountered in some derivation starting from query $G$; this
is called the \emph{outside} probability of $G'$ w.r.t. $G$.  In the
program in Fig.~\ref{fig:hmm-prism}, the outside probability of
\texttt{hmm\_part(i,n,S,T)} for different values of \texttt{S}
w.r.t. initial query \texttt{hmm(n,T)} gives the filter distribution
of the \texttt{i}-th state.  This distribution, multiplied with the
inside probability for query \texttt{hmm\_part(i,n,S,T)} gives the
\emph{smoothed} distribution of the \texttt{i}-th state.
}



\section{Extended PRISM}
\label{sec:language}

Support for continuous variables is added by modifying PRISM's
language in two ways.   We use the \texttt{msw}
relation to sample from discrete as well as continuous distributions.
In PRISM, a special relation
called \texttt{values} is used to specify the ranges of values of
random variables; the probability mass functions are specified using
\texttt{set\_sw} directives.  In our extension, we extend the
\texttt{set\_sw} directives to specify probability density functions
as well.  For instance, \texttt{set\_sw(r, norm(Mu,Var))} specifies
that outcomes of  random processes \texttt{r} have Gaussian
distribution with mean 
\texttt{Mu} and variance \texttt{Var}\footnote{The technical
  development in this paper considers only univariate Gaussian
  variables; see Discussions section on a discussion on how
  multivariate Gaussian as well as other continuous distributions are
  handled.}.  Parameterized families of random
processes may be specified, as long as the parameters are
discrete-valued.  For 
instance, \texttt{set\_sw(w(M), norm(Mu,Var))} specifies a family of
random processes, with one for each value of \texttt{M}. As
in PRISM, \texttt{set\_sw} directives may be specified
programmatically; for instance, in the specification of \texttt{w(M)},
the distribution parameters may be computed as functions of
\texttt{M}.

Additionally, we extend PRISM programs with linear equality
constraints over reals.  Without loss of generality, we
assume that constraints are written as linear equalities of the form
$Y = a_1 * X_1 + \ldots + a_n * X_n + b$ where $a_i$ and $b$ are all
floating-point constants.
\comment{ , or inequalities of the form $Y < a$ or $Y>a$
for some floating point constant $a$. 
Note that inequalities comparing
two variables can be expressed as an inequality comparing a linear
function of the two variables and a constant.}
The use of constraints enables us to encode Hybrid Bayesian Networks and Kalman Filters as
extended PRISM programs.  In the following, we use \emph{Constr} to
denote a set (conjunction) of linear equality 
constraints.  We also denote by $\overline{X}$ a vector of variables
and/or values, explicitly specifying the size only when it is not
clear from the context.  This permits us to write linear equality
constraints compactly (e.g., $Y = \overline{a}\cdot\overline{X} +b$).
\comment{

\begin{Ex}[Hybrid Bayesian Network~\cite{hbn}]
In the following PRISM program, which encodes a Hybrid Bayesian
Network, \texttt{R} and \texttt{S}  are valuations of random variables
\texttt{r} and \texttt{s}, respectively, such that \texttt{r} is the
parent of \texttt{s}.   The conditional dependency between \texttt{s}
and \texttt{r} is specified by generating a value for \texttt{s} as
as a linear combination of \texttt{r}'s value
and a noise term.

\begin{small}
  \begin{minipage}{3.0in}
\begin{verbatim}
hbn(R, S) :- msw(r, R),
             msw(noise, E),
             S = R + E.

% Ranges of RVs
values(r, real).
values(noise, real).
% PDFs and PMFs:
:- set_sw(r, norm(1.0, 0.1)),
   set_sw(noise, norm(0.0, 0.2)).
\end{verbatim}
\end{minipage}
\end{small}
\qed
  \label{ex:hbn}
\end{Ex}
}

Encoding of Kalman Filter specifications
uses
linear constraints and closely follows the structure of the HMM specification,
and is shown in Section~\ref{sec:kalman}.

\comment{
\begin{Ex}[Finite Mixture Model]
In the following PRISM program, which encodes a finite mixture model~\cite{fmm},
{\rm \texttt{msw(m, M)}} chooses one distribution from a finite set of
continuous 
distributions, {\rm \texttt{msw(w(M), X)}} samples {\rm \texttt{X}} from the chosen distribution.

\begin{small}
  \begin{minipage}{3.0in}
\begin{verbatim}
fmix(X) :- msw(m, M),
           msw(w(M), X).

% Ranges of RVs
values(m, [a,b]).
values(w(M), real).
% PDFs and PMFs
:- set_sw(m, [0.3, 0.7]),
   set_sw(w(a), norm(1.0, 0.2)),
   set_sw(w(b), norm(1.05, 0.1)).
\end{verbatim}
\end{minipage}
\end{small}
\qed
  \label{ex:fmm}
\end{Ex}
}

\paragraph{Distribution Semantics:}
We extend PRISM's distribution semantics for continuous random
variables as follows.  The idea is to construct a probability space
for the \texttt{msw} definitions (called probabilistic facts in PRISM) and then
extend it to a probability space for the entire program using least
model semantics.  Sample space for the probabilistic facts is constructed from
those of discrete and continuous random variables. The sample space of
a continuous random variable is the set of real numbers, $\Re$.  The
sample space of a set of random variables is a Cartesian product of
the sample spaces of individual variables.  We complete the definition
of a probability space for $N$ continuous random variables by
considering the Borel $\sigma$-algebra over $\Re^N$, and defining a
Lebesgue measure on this set as the probability measure.  Lifting the
probability space to cover the entire program needs one significant
step.  We use the least model semantics of constraint logic
programs~\cite{JaffarCLP} as the basis for defining the semantics of
extended PRISM programs.  A point in the sample space is an arbitrary
interpretation of the program, with its Herbrand universe and $\Re$ as the
domain of interpretation.  For each sample, we distinguish between the
interpretation of user-defined predicates and probabilistic facts.  Note that
only the probabilistic facts have probabilistic behavior in PRISM; the rest of
a model is defined in terms of logical consequence.  Hence, we can
define a probability measure over a set of sample points by using the
measure defined for the probabilistic facts alone.  The semantics of an
extended PRISM program is thus defined as a distribution over its
possible models.


\section{Inference}
\label{sec:inference}

\newcommand{\marginalize}{\mathbb{M}}
\newcommand{\integrate}{\mathbb{I}}
\newcommand{\project}{\mathbb{P}}

Recall that PRISM's inference explicitly enumerates outcomes of random
variables in derivations.  The key to inference in the presence of
continuous random variables is avoiding enumeration by 
representing the derivations and their attributes symbolically.  A single
step in the construction of a symbolic derivation is defined below.

\begin{Def}[Symbolic Derivation]
A goal $G$ \emph{directly derives}  goal $G'$, denoted $G \rightarrow
G'$,  if:
\begin{description}
\item[PCR:] $G = q_1(\overline{X_1}), G_1$, and there exists a
  clause in the program,\\ $q_1(\overline{Y}) :- r_1(\overline{Y_1}),
  r_2(\overline{Y_2}), \ldots, r_m(\overline{Y_m})$, such that $\theta =
  \mgu(q_1(\overline{X_1}), q_1(\overline{Y}))$; then, $G' =
  (r_1(\overline{Y_1}), r_2(\overline{Y_2}), \ldots,
  r_m(\overline{Y_m}), G_1)\theta$; 
\item[MSW:] $G = \mathtt{msw}(\rv(\overline{X}), Y), G_1$: then
  $G' = G_1$;
\item[CONS:]$G = \id{Constr}, G_1$ and $\id{Constr}$ is satisfiable: then $G' = G_1$.
\end{description}
A \emph{symbolic derivation} of $G$ is a sequence of goals $G_0, G_1,
\ldots$ such that $G=G_0$ and, for all $i \geq 0$, $G_i \rightarrow G_{i+1}$.
\end{Def}
We only consider successful derivations, i.e., the last step of a derivation resolves to an empty clause.
Note that the traditional notion of derivation in a logic program
coincides with that of symbolic derivation when the selected subgoal
(literal) is not an \texttt{msw} or a constraint. When the selected
subgoal is an \texttt{msw}, PRISM's inference will construct the next
step by enumerating the values of the random variable.  In contrast,
symbolic derivation skips \texttt{msw}'s and constraints and
continues with the remaining subgoals in a goal.  The effect of these
constructs is computed by associating (a) variable type information
and (b) a success function (defined below) with each goal in the
derivation.  The symbolic derivation for the goal \texttt{widget(X)} over the program
in Example~\ref{ex:mm} is shown in
Fig.~\ref{fig:fmmder}.

\begin{figure}[h]
  \centering
  \begin{subfigure}[b]{0.4\textwidth}
    \includegraphics{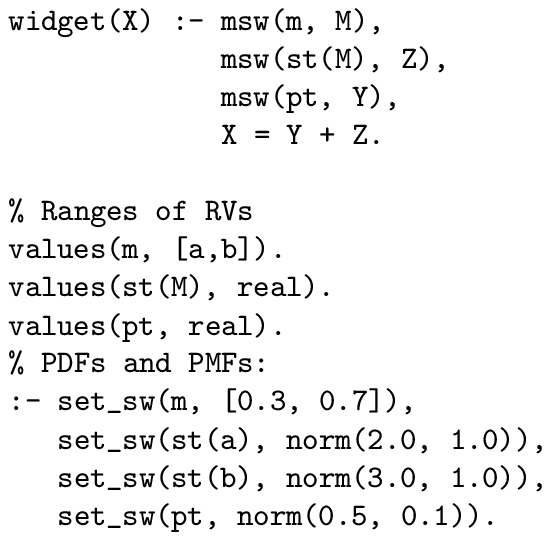}
   \caption{Mixture model program}
    \label{fig:fmmcode}
  \end{subfigure}
 \quad
  \begin{subfigure}[b]{0.55\textwidth}
    \hspace*{-.2in}
   \includegraphics{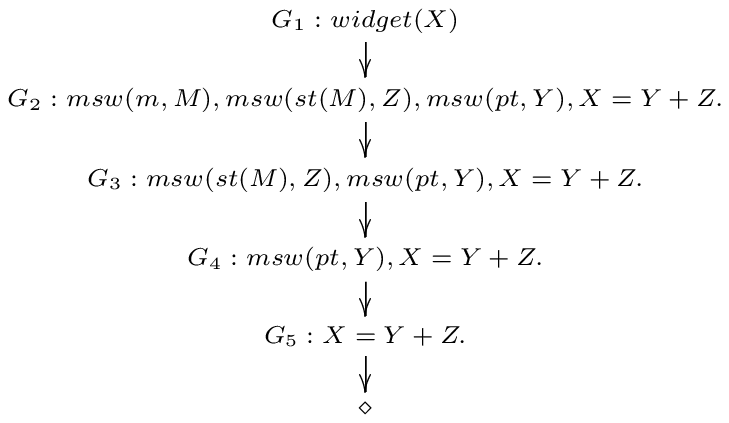}
   \caption{Symbolic derivation for goal \texttt{widget(X)}}
    \label{fig:fmmder}
  \end{subfigure}
   \caption{Finite Mixture Model Program and Symbolic Derivation}
  \label{fig:fmm}
\end{figure}
\vspace{-5pt}

\begin{Ex}
Consider a factory with two machines \texttt{a} and \texttt{b}.
Each machine produces a widget structure and then the structure is painted with a color.
In the program Fig.~\ref{fig:fmmcode},
\texttt{msw(m, M)} chooses either machine \texttt{a} or \texttt{b}, 
\texttt{msw(st(M), Z)} gives the cost \texttt{Z} of a product structure,
\texttt{msw(pt, Y)} gives the cost  \texttt{Y} of painting, and 
finally \texttt{X = Y + Z} returns the price of a painted widget \texttt{X}.
\qed
  \label{ex:mm}
\end{Ex}

\begin{figure}[h]
  \begin{subfigure}[t]{0.28\textwidth}
    \centering
    \includegraphics{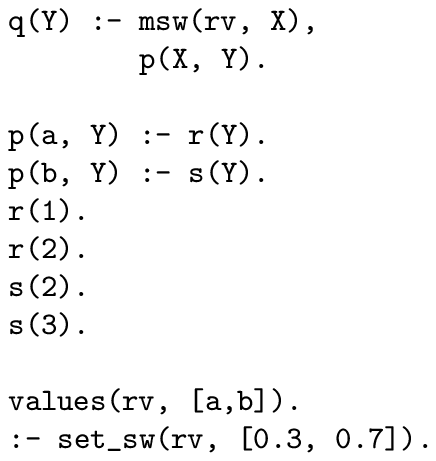}
    \caption{Example program}
    \label{fig:symcode}
  \end{subfigure}
 ~~~~~~~~~~~~~~~
  \begin{subfigure}[t]{0.45\textwidth}
    \centering
    \includegraphics{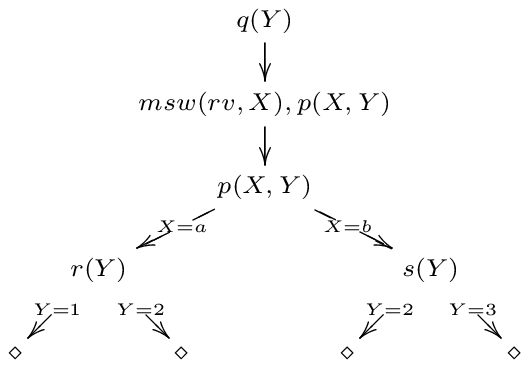}
    \caption{Symbolic derivation for goal \texttt{q(Y)}}
    \label{fig:symder}
  \end{subfigure}
  \caption{Symbolic derivation}
  \label{fig:symcd}
\end{figure}
\vspace{-10pt}

\begin{Ex}
\label{ex:symder}
This example illustrates how symbolic derivation differs from traditional logic programming derivation.
Fig.~\ref{fig:symder} shows the symbolic derivation for goal $q(Y)$ in Fig.~\ref{fig:symcode}. 

Notice that the symbolic derivation still makes branches in the derivation
tree for various logic definitions and outcomes. But the main difference with traditional logic derivation is that it skips \texttt{msw} and \texttt{Constr} definitions, and continues with the remaining subgoals in a goal.
\qed
\end{Ex}

\paragraph{Success Functions:}
Goals in a symbolic derivation may contain variables whose values are
determined by \texttt{msw}'s appearing subsequently in the
derivation.  With each goal $G_i$ in a symbolic derivation, we
associate a set of variables, $V(G_i)$, that is a subset of variables
in $G_i$.  The set $V(G_i)$ is such that the variables in $V(G_i)$ 
subsequently appear as parameters or
outcomes of \texttt{msw}'s in some subsequent  goal $G_j$, $j \geq i$.
We can further partition $V$ into two disjoint sets, $V_c$ and $V_d$,
representing continuous and discrete variables,
respectively.  The sets $V_c$ and $V_d$ are called the derivation
variables of $G_i$, defined below.
\begin{Def}[Derivation Variables]
Let $G \rightarrow G'$ such that $G'$ is derived from $G$ using:
\begin{description}
\item[PCR:] Let $\theta$ be the mgu in this step.  Then $V_c(G)$ and
  $V_d(G)$ are the
  largest sets of variables in $G$ such that $V_c(G)\theta \subseteq
 V_c(G')$ and $V_d(G)\theta \subseteq V_d(G')$.
\item[MSW:] Let $G= \kw{msw}(\rv(\overline{X}), Y), G'$.  Then $V_c(G)$
  and $V_d(G)$ are the largest sets of variables in $G$ such that
  $V_c(G) \subseteq V_c(G') \cup \{Y\}$, and $V_d(G) \subseteq V_d(G')
  \cup \overline{X}$ if $Y$ is continuous, 
  otherwise $V_c(G) \subseteq V_c(G')$, and $V_d(G) \subseteq V_d(G')
  \cup \overline{X}  \cup \{Y\}$.
\item[CONS:] Let $G= \id{Constr}, G'$.  Then $V_c(G)$
  and $V_d(G)$ are the largest sets of variables in $G$ such that
  $V_c(G) \subseteq V_c(G') \cup \id{vars}(\id{Constr})$, and
  $V_d(G) \subseteq V_d(G')$.
\end{description}
\end{Def}

Given a goal $G_i$ in a symbolic derivation, we can associate with it a
\emph{success function}, which is a function from the set of all
valuations of $V(G_i)$ to $[0,1]$.  Intuitively, the success function
represents the probability that the symbolic derivation represents a
successful derivation for each valuation of $V(G_i)$.

\paragraph{Representation of success functions:}
Given a set of variables $\mathbf{V}$, let $\mathbf{C}$ denote the set
of all linear equality constraints over reals using
$\mathbf{V}$.  Let $\mathbf{L}$ be the set of all linear functions over
$\mathbf{V}$ with real coefficients.  Let
$\mathcal{N}_{X}(\mu,\sigma^2)$ be the PDF of a univariate Gaussian
distribution with mean $\mu$ and variance $\sigma^2$, and $\delta_{x}(X)$ be the Dirac delta function which is zero everywhere except at $x$ and integration of the delta function over its entire range is 1.  Expressions of
the form $k*\prod_{l} \delta_{v}(V_{l}) \prod_{i} \mathcal{N}_{f_i}$, where $k$ is a non-negative
real number and $f_i \in \mathbf{L}$, are called \emph{product PDF
  (PPDF) functions 
  over} $\mathbf{V}$.  We use $\phi$ (possibly subscripted) to denote
such functions.  A pair $\langle \phi, C\rangle$ where $C \subseteq \mathbf{C}$ is
called a \emph{constrained PPDF function}. A sum of a finite number of
constrained PPDF functions is called a \emph{success function}, represented as
$\sum_{i} \langle \phi_{i}, C_{i} \rangle$.

We use $C_{i}(\psi)$ to denote the
constraints (i.e., $C_{i}$) in the $i^{th}$ constrained PPDF function of
success function $\psi$;
and $D_{i}(\psi)$ to denote the $i^{th}$ PPDF
function of $\psi$.

\paragraph{Success functions of base predicates:}
The success function of a constraint $C$ is $\langle 1, C\rangle$.
The success function of \emph{true} is $\langle 1, \id{true}\rangle$.
The PPDF component of $\mathtt{msw}(\rv(\overline{X}), Y)$'s success
function  is the
probability density function of $\rv$'s distribution if $\rv$ is
continuous, and its probability mass function if $\rv$ is discrete; its constraint component
is \id{true}.

\begin{wrapfigure}[6]{r}{0.40\textwidth}
\begin{center}
\includegraphics[width=0.38\textwidth]{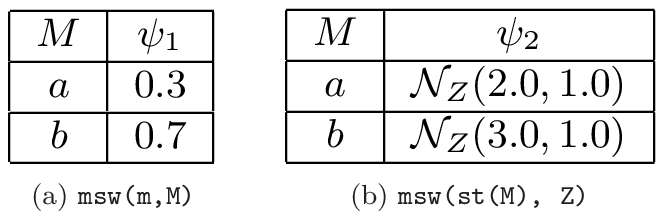}
 \caption{Success Functions}
 \label{fig:basesf}
\end{center}
\end{wrapfigure}

\begin{Ex}
The success function $\psi_1$ of \texttt{msw(m,M)} for the program in
Example~\ref{ex:mm} is such that $\psi_1 = 0.3 \delta_{a}(M) + 0.7 \delta_{b}(M)$.
Note that we can represent the success function using tables, where each table row denotes 
discrete random variable valuations. For example, the above 
success function can be represented as Fig.~\ref{fig:basesf}a.
Thus instead of using delta functions, we often omit it in examples and represent success functions using tables.

Fig.~\ref{fig:basesf}b represents the success function of \texttt{msw(st(M), Z)} for the program in
Example~\ref{ex:mm}.
Similarly, the success function $\psi_3$ of \texttt{msw(pt, Y)} for the program in
Example~\ref{ex:mm} is $\psi_3 =  \mathcal{N}_{Y}(0.5, 0.1)$.

Finally, the success function $\psi_4$ of \texttt{X = Y + Z} for the program in
Example~\ref{ex:mm} is $\psi_4 =  \cpdf{1}{X = Y + Z}$.
\qed
\end{Ex}

\paragraph{Success functions of user-defined predicates:}
If $G \rightarrow G'$ is a step in a derivation, then the success
function of $G$ is computed bottom-up based on the success
function of $G'$.  This computation is done using \emph{join} and
\emph{marginalize} operations on success functions.

\begin{Def} [Join]
Let $\psi_{1}=\sum_{i} \langle D_{i}, C_{i} \rangle$ and $\psi_{2} =
\sum_{j} \langle D_{j}, C_{j} \rangle$ be two success functions, then
join of $\psi_{1}$ and $\psi_{2}$ represented as $\psi_{1} * \psi_{2}$
is the success function $\sum_{i,j} \langle D_{i}D_{j}, C_{i} \wedge
C_{j} \rangle$. 
\end{Def}

\begin{Ex}
\label{join-ex}
Let Fig.~\ref{fig:sf1} and~\ref{fig:sf2} represent the success functions $\psi_{msw(m, M)}(M)$ and $\psi_{G_3}(X, Y, Z, M)$ respectively.

\begin{figure}[h]
  \centering
  \begin{subfigure}[t]{0.25\textwidth}
    \includegraphics{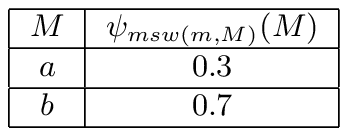}
    \caption{}
    \label{fig:sf1}
  \end{subfigure}
  ~~~~~~~~
\begin{subfigure}[t]{0.50\textwidth}
   \includegraphics{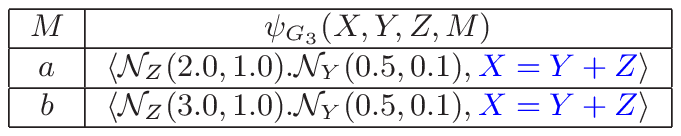}
   \caption{}
   \label{fig:sf2}
\end{subfigure}

  \begin{subfigure}[t]{0.6\textwidth}
    \includegraphics{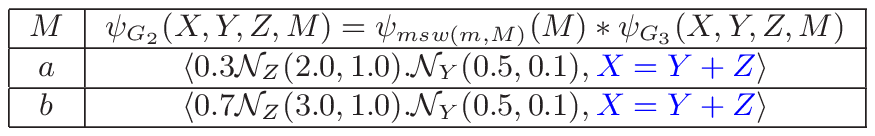}
    \caption{}
    \label{fig:sf3}
  \end{subfigure}
   \caption{Join of Success Functions}
  \label{fig:joinsf}
\end{figure}
\vspace{-5pt}

Then Fig.~\ref{fig:sf3} shows the join of $\psi_{msw(m, M)}(M)$ and $\psi_{G_3}(X, Y, Z, M)$.
\qed
\end{Ex}

Note that we perform a simplification of success functions after the join operation. 
We eliminate any PPDF term in $\psi$ which is inconsistent w.r.t. delta functions. 
For example, $\delta_{a}(M)\delta_{b}(M)=0$ as $M$ can not be both $a$ and $b$ at the same time.

Given a success function $\psi$ for a goal $G$, the success function
for $\exists X.\ G$ is computed by the marginalization operation.
Marginalization w.r.t. a discrete variable is straightforward and
omitted.  Below we
define marginalization w.r.t. continuous variables in two steps: first
rewriting the success function in a projected form and then doing the
required integration.

The goal of projection is to eliminate any linear constraint on $V$, where $V$ is the continuous variable to marginalize over. The projection operation involves finding a linear constraint 
(i.e., $V = \overline{a} \cdot \overline{X}+b$) on $V$ and replacing all occurrences of $V$ in the success function by  $\overline{a} \cdot \overline{X}+b$.

\begin{Def} [Projection]
Projection of a success function $\psi$ w.r.t. a
continuous variable $V$,  denoted by $\psi\downarrow_{V}$, is a
success function $\psi'$ such that\\
$\forall i.\ D_{i}(\psi') =
D_{i}(\psi)[\overline{a} \cdot \overline{X}+b / V]$; and 
$C_{i}(\psi') = (C_{i}(\psi) - C_{ip})[\overline{a} \cdot \overline{X}+b / V]$,\\
where $C_{ip}$ is a linear constraint ($V = \overline{a} \cdot
  \overline{X}+b$) on $V$ in $C_{i}(\psi)$ and
  $t[s/x]$ denotes replacement of all occurrences of $x$ in $t$ by $s$.
\end{Def}

Note that the replacement of $V$ by $\overline{a} \cdot \overline{X}+b$ in 
PDFs and linear constraints does not alter the general form of a success function. Thus projection returns a 
success function. Notice that if $\psi$ does not contain any linear constraint on $V$, then the projected form remains the same.

\begin{Ex}
Let $\psi_{1} = \cpdf{0.3  \mathcal{N}_{Z}(2.0, 1.0).\mathcal{N}_{Y}(0.5, 0.1)}{X = Y + Z}$ represent a success function. 
Then projection of $\psi_{1}$ w.r.t. $Y$ yields  
\begin{equation}
\label{prj-eq}
 \psi_{1}\downarrow_{Y} = 0.3  \mathcal{N}_{Z}(2.0, 1.0).\mathcal{N}_{X - Z}(0.5, 0.1).
\end{equation}
 Notice that $Y$ is replaced by $X - Z$. 
 \qed
 \end{Ex}

\begin{Pro}
\label{integration}
Integration of a PPDF function with respect to a variable $V$ is a PPDF function, i.e., 
\begin{align*}
\alpha \int^{\infty}_{-\infty} \prod_{k=1}^{m}
\mathcal{N}_{(\overline{a_{k}} \cdot
  \overline{X_{k}}+b_{k})}(\mu_{k},\sigma^{2}_{k}) dV = \alpha' \prod_{l=1}^{m'}
\mathcal{N}_{(\overline{a'_{l}} \cdot
  \overline{X'_{l}})+b'_{l}}(\mu'_{l},\sigma'^{2}_{l})
\end{align*}
where $V \in \overline{X_{k}}$ and $V \not\in \overline{X'_{l}}$.
\end{Pro}

For example, 
\comment{the integration of $\mathcal{N}_{a_{1} V-X_{1}}(\mu_{1}, \sigma_{1}^{2}) . \mathcal{N}_{a_{2} V-X_{2}}(\mu_{2}, \sigma_{2}^{2})$ w.r.t. variable $V$ is }
\begin{align}
\label{intgr-eq}
\int^{\infty}_{-\infty} \mathcal{N}_{a_{1} V - X_{1}} (\mu_{1}, \sigma_{1}^{2}) . \mathcal{N}_{a_{2} V - X_{2}}(\mu_{2}, \sigma_{2}^{2}) dV \notag \\ 
     = \mathcal{N}_{a_{2} X_{1} - a_{1} X_{2}} (a_{1} \mu_{2} - a_{2}\mu_{1} , a_{2}^{2} \sigma_{1}^{2} + a_{1}^{2} \sigma_{2}^{2}).
\end{align}
Here $X_{1}, X_{2}$ are linear combinations of variables (except $V$).
A proof of the proposition is presented in Section~\ref{sec:ppdfint}.

\begin{Def} [Integration]
\label{intgop}
Let $\psi$ be a success function that does not contain any linear constraints on $V$. Then integration of $\psi$ with
 respect to $V$, denoted by $\oint_{V}\psi$ is a
success function $\psi'$ such that
$\forall i. D_{i}(\psi') = \int D_{i}(\psi) dV$.
\end{Def}

It is easy to see (using Proposition~\ref{integration}) that the
integral of success functions are also success functions.
Note that if $\psi$ does not contain any PDF on $V$, then the integrated form remains the same.

\begin{Ex}
Let $\psi_{2} = 0.3  \mathcal{N}_{Z}(2.0, 1.0).\mathcal{N}_{X - Z}(0.5, 0.1)$ represent a success function.
Then integration of $\psi_{2}$ w.r.t. $Z$ yields
\begin{align}
\label{int-eq}
\oint_{Z}\psi_{2} &=  \int 0.3  \mathcal{N}_{Z}(2.0, 1.0).\mathcal{N}_{X - Z}(0.5, 0.1)dZ \notag\\
   &=0.3  \mathcal{N}_{X}(2.5, 1.1).  \textrm{   (using Equation~\ref{intgr-eq})} \notag
\end{align}
 \end{Ex}

\begin{Def} [Marginalize]
Marginalization of a success function $\psi$ with respect to a
variable $V$, denoted by $\marginalize(\psi, V)$, is a success
function $\psi'$ such that
\begin{align*}
\psi'  &= \oint_{V} \psi \downarrow_{V} 
\end{align*}
\end{Def}

We overload $\marginalize$ to denote marginalization over a set of
variables, defined such that
$\marginalize(\psi, \{V\} \cup \overline{X}) =
\marginalize(\marginalize(\psi, V), \overline{X})$ and
$\marginalize(\psi, \{\}) = \psi$.

\begin{Pro}
\label{opclosure}
The set of all success functions
is closed under join and marginalize operations.
\end{Pro}

The success function for a derivation is defined as
follows.

\comment{
\begin{Def}[Success function of a derivation]
\label{sfder}
 Let $G \rightarrow G'$.  Then the success function of $G$, denoted by
 $\psi_G$, is computed from that of $G'$, based on the way $G'$ was derived:
\begin{description}
\item[PCR:] $\psi_G = \marginalize(\psi_{G'}, V(G') - V(G))$.
\item[MSW:] Let $G= \kw{msw}(\rv(\overline{X}), Y), G_1$.
  Then $\psi_G = \psi_{msw(rv(\overline{X}),Y)} * \psi_{G'}$.
\item[CONS:] Let $G= Constr, G_1$.  Then
  $\psi_G = \psi_{Constr} * \psi_{G'}.$
\end{description}
\end{Def}
}

\begin{Def}[Success function of a goal]
\label{sfder}
The success function of a goal $G$, denoted by
$\psi_G$, is computed based on the derivation $G \rightarrow G'$:
 \[
\psi_G  = \left\{
\begin{array}{ll}
\sum_{G'}\marginalize(\psi_{G'}, V(G') - V(G)) & \text{for all program clause resolution $G \rightarrow G'$}\\
\psi_{msw(rv(\overline{X}),Y)} * \psi_{G'} & \text{if $G= \kw{msw}(\rv(\overline{X}), Y), G'$}\\
\psi_{Constr} * \psi_{G'}  & \text{if $G= Constr, G'$}\\
\end{array}\right.
\]
\end{Def}

Note that the above definition carries PRISM's assumption that an
instance of a random variable occurs at most once in any derivation.
In particular, the PCR step marginalizes success functions w.r.t. a
set of variables; the valuations of the set of variables must be 
mutually exclusive for correctness of this step.  The MSW step joins
success functions; the goals joined must use independent random
variables for the join operation to correctly compute success
functions in this step.

\begin{Ex}
\label{exmmjm}
  Fig.~\ref{fig:fmmder} shows the symbolic derivation
  for the goal \texttt{widget(X)} over the mixture model program in
  Example~\ref{ex:mm}.  The success function of goal $G_{5}$ is $  \psi_{G_{5}}(X, Y, Z) = \cpdf{1}{X = Y + Z}$. 
  \begin{align*}
   \psi_{G_{4}}(X, Y, Z) = \psi_{msw(pt, Y)}(Y) * \psi_{G_{5}}(X, Y, Z) 
   = \cpdf{ \mathcal{N}_{Y}(0.5, 0.1)}{X = Y + Z}.
  \end{align*}
  
  The success function of goal $G_{3}$ is $\psi_{msw(st(M), Z)}(Z)*
  \psi_{G_{4}}(X, Y, Z)$ (Fig.~\ref{fig:sf2}).
  
 \comment{
  \begin{figure}[h]
  \centering
  \begin{subfigure}[b]{.48\textwidth}
    \includegraphics[width=1.0\textwidth]{table6}
    \caption{}
    \label{fig:sfg3}
  \end{subfigure}
 \begin{subfigure}[b]{.48\textwidth}
    \includegraphics[width=1.0\textwidth]{table7}
    \caption{}
    \label{fig:sfg2}
  \end{subfigure}
  \caption{Success Functions}
  \label{fig:mmjoin}
\end{figure}
}

Then join of $\psi_{msw(m, M)}(M)$ and $\psi_{G_{3}}(X, Y, Z, M)$ yields
 the success function in Fig.~\ref{fig:sf3}
(see Example~\ref{join-ex}).

  Finally, $\psi_{G_1}(X) = \marginalize(\psi_{G_2}(X, Y, Z, M),  \{M, Y, Z\})$. 
  
  First we marginalize $\psi_{G_2}(X, Y, Z, M)$ w.r.t. $M$: 
 \begin{align*}
 \psi_{G_2}' &= \marginalize(\psi_{G_2}, M) = \oint_{M} \psi_{G_2} \downarrow_{M} \\
  &= \cpdf{0.3 \mathcal{N}_{Z}(2.0, 1.0).\mathcal{N}_{Y}(0.5, 0.1)}{X = Y + Z}\\ &+ 
  \cpdf{0.7 \mathcal{N}_{Z}(3.0, 1.0).\mathcal{N}_{Y}(0.5, 0.1)}{X = Y + Z}.
 \end{align*}
 
 Next we marginalize the above success function w.r.t. $Y$:
 \begin{align*}
 \psi_{G_2}'' &= \marginalize(\psi_{G_2}', Y) = \oint_{Y} \psi_{G_2}' \downarrow_{Y} \\
  &= 0.3  \mathcal{N}_{Z}(2.0, 1.0).\mathcal{N}_{X - Z}(0.5, 0.1) + 0.7  \mathcal{N}_{Z}(3.0, 1.0).\mathcal{N}_{X - Z}(0.5, 0.1).
 \end{align*}
 
 Finally, we marginalize the above function over variable $Z$ to get $\psi_{G_1}(X)$:
 \begin{align*}
 \psi_{G_1}(X) = \marginalize(\psi_{G_2}'', Z)
  =  \oint_{Z} \psi_{G_2}'' \downarrow_{Z}
  = 0.3  \mathcal{N}_{X}(2.5, 1.1) + 0.7  \mathcal{N}_{X}(3.5, 1.1).
 \end{align*} 
\end{Ex}

\begin{Ex}
In this example, we compute success function of goal $q(Y)$ in Example~\ref{ex:symder}. 
Fig.~\ref{fig:symder} shows the symbolic derivation for goal $q(Y)$.
Success function of $r(Y)$ is $\delta_{1}(Y) +  \delta_{2}(Y)$, and 
success function of $s(Y)$ is $\delta_{2}(Y) +  \delta_{3}(Y)$.
Similarly, success function of $p(X,Y)$ is $\delta_{a}(X)(\delta_{1}(Y) +  \delta_{2}(Y)) + \delta_{b}(X)(\delta_{2}(Y) +  \delta_{3}(Y))$.
Now 
\begin{align*}
\psi_{q(Y)} &= \marginalize(\psi_{msw(rv, X)} * \psi_{p(X,Y)}, X)
\end{align*}
Success function of $msw(rv,X)$ is $0.3 \delta_{a}(X) + 0.7 \delta_{b}(X)$.
Join of $\psi_{msw(rv, X)}$ and $\psi_{p(X,Y)}$ yields
$0.3 \delta_{a}(X)(\delta_{1}(Y) +  \delta_{2}(Y)) + 0.7 \delta_{b}(X)(\delta_{2}(Y) +  \delta_{3}(Y))$.
Finally, $\psi_{q(Y)} = 0.3 (\delta_{1}(Y) +  \delta_{2}(Y)) + 0.7 (\delta_{2}(Y) +  \delta_{3}(Y))$.

When $Y=1$, only $p(a, 1)$ is true. Thus $\psi_{q(1)} = 0.3$. 
On the other hand, $\psi_{q(2)} = 1.0$ as both $p(a, 2)$ and $p(b,2)$ are true when Y=2.
Similarly,  $\psi_{q(3)} = 0.7$.
\qed
\end{Ex}

\paragraph{Complexity:}

Let $S_{i}$ denote the number of constrained PPDF terms in $\psi_i$;
$P_{i}$ denote the maximum number of product terms in any PPDF function
in $\psi_i$; and
$Q_{i}$ denote the maximum size of a constraint set ($C_{i}$) in
$\psi_i$. The time complexity of the two basic operations used in
constructing a symbolic derivation is as follows.

\begin{Pro}[Time Complexity]
The worst-case time complexity of $\id{Join}(\psi_{i},\psi_{j})$ is
$O(S_{i}*S_{j}*(P_{i}*P_{j}+Q_{i}*Q_{j}))$.

The worst-case time complexity of $\marginalize(\psi_{g}, V)$ is
$O(S_{g}*P_{g})$ when $V$ is discrete and $O(S_{g}*(P_{g}+Q_{g}))$ when $V$ is continuous.
\end{Pro}

Note that when computing  the success function of a goal in a
derivation, the join operation is limited to joining the success
function of a single
\texttt{msw} or a single constraint set to the success function of a
goal, and hence the parameters $S_i,P_i$, and $Q_i$ are typically small.
The complexity of the \emph{size} of success functions is as follows.

\begin{Pro}[Success Function Size]
  For a goal $G$ and its symbolic derivation, the following hold:
  \begin{enumerate}
 \vspace{-2pt}
 \item The maximum number of product terms in any PPDF function in $\psi_G$
    is linear in $|V_c(G)|$, the number of continuous variables in $G$.
  \item The maximum size of a constraint set in a constrained PPDF
    function in $\psi_G$ is linear in $|V_c(G)|$.
  \item The maximum number of constrained PPDF functions in
    any entry of $\psi_G$ is potentially exponential in the number of
    discrete random variables in the symbolic derivation.
    \vspace{-2pt}
  \end{enumerate}
\end{Pro}

The number of product terms and the size of constraint sets are hence
independent of the length of the symbolic derivation.  Note that for a
program with only discrete random variables, there may be
exponentially fewer symbolic derivations than concrete derivations.
The compactness is only in terms of \emph{number}
of derivations and not the total size of the representations.  In
fact, for programs with only discrete random variables, there is a
one-to-one correspondence between the entries in the tabular
representation of success functions and PRISM's answer tables.  For
such programs, it is easy to show that the time complexity of the
inference algorithm presented in this paper is same as that of PRISM.

\paragraph{Correctness of the Inference Algorithm:}

The technically complex aspect of correctness is the closure of the set of success functions w.r.t. join and marginalize operations. Proposition \ref{integration} and \ref{opclosure} state these closure properties.
Definition \ref{sfder} represents the inference algorithm for computing the success function of a goal. 
The distribution of a goal is formally defined in terms of the distribution semantics of extended PRISM programs and is computed using the inference algorithm. 

\begin{Thm}
The success function of a goal computed by the inference algorithm represents the distribution of the answer to that goal.
\end{Thm}

\paragraph{Proof:}
Correctness w.r.t. distribution semantics follows from the definition of join and marginalize operations, and PRISM's independence and exclusiveness assumptions.
We prove this by induction on derivation length $n$.
For $n=1$, the definition of success function for base predicates gives a correct distribution. 

Now let's assume that for a derivation of length $n$, our inference algorithm computes valid distribution.
Let's assume that $G'$ has a derivation of length $n$ and $G \rightarrow G'$. Thus $G$ has a derivation of length $n+1$. We show that the success function of $G$ represents a valid distribution. 

We compute $\psi_{G}$ using Definition \ref{sfder} 
and it carries PRISM's assumption that an 
instance of a random variable occurs at most once in any derivation.
More specifically, the PCR step marginalizes $\psi_{G'}$ w.r.t. a
set of variables $V(G') - V(G)$. Since according to PRISM's exclusiveness 
assumption the valuations of the set of variables are
mutually exclusive, the marginalization operation returns a valid distribution. 
Analogously, the MSW/CONS step joins success functions, and 
the goals joined use independent random
variables (following PRISM's assumption) for the join operation to correctly compute 
$\psi_{G}$ in this step. Thus $\psi_{G}$ represents a valid distribution.


\comment{ Summation entries are introduced to $\psi_{g}$ whenever we
  do a marginalization over a discrete variable, or join of two
  success functions (where we multiply two entries and get a number of
  summation entries). Since every derivation step can involve either
  of these two operations, size of summation entries is then
  proportional to the size of derivation.

Let, there are $n$ continuous variables in goal $g$, then $D_{i}(\psi_{g} \langle e \rangle)$ will contain approximately $n$ Gaussians. If it contains more than $n$ Gaussians, then some continuous variable $V$ has more than one Gaussian, and multiplying those will yield a single Gaussian according to Gaussian functions properties.

Rough Proof sketch (by induction). Let, in the base case, $q$ is defined using only some constraint or other non-predicates. Then, the size of constraints set of $q$ is equal to the total number of constraints appeared in the definition of $q$.

Let's assume that $q$ is defined in terms of predicates $r$, and $s$. ($e.g., q(X):-r(X,Y),s(X,Z)$). As an induction hypothesis, let's assume that the size of constraint sets in summation entries of $\psi_{r}$ and $\psi_{s}$ are proportional to the program size of $r$ and $s$.
Since projection or integration operation will always reduce the constraints set size, the number of constraints in $q$ after marginalizing over the remaining variables will be at most the combined program size of $r$ and $s$.
}

\comment{
Notice that, success function is similar to PRIMS's inside probability ~\cite{sato} of a goal.
In our formulation, success functions are represented by probability tables,
and computation behind success function lies in table joining and marginalization.

\subsection{CALL FUNCTION}
We associate a function with each goal's `Call' in a derivation, which denotes the probability that the goal will be called given the derivation context (i.e., variable bindings). Call function corresponds to prior or filtered distribution of random variables in a goal. It is similar in notion to PRISM's outside probability ~\cite{sato} of a goal.

Call function computation is a top-down procedure, and we compute call
functions with the help of magic-sets \cite{magicset}. Any logic
program can be transformed to its corresponding magic-sets, and call
function of a goal in original program is same as success function of
that goal in magic-sets. We denote the call function of goal
$g(\overline{X_{g}})$ by $\psi_{g}^{m}(\overline{X_{g}})$ where the
superscript $m$ means that it's the success function of $g$ in
magic-sets. Next section describes call and success function
computation with an illustrative example.
}

\comment{
\subsection{PROBABILISTIC PREDICATES}

\textbf{Probabilistic switch, \emph{msw}.}\\
We introduce a probabilistic switch in logic programs which is same as PRISM's \emph{msw} except that it introduces continuous random variable definitions. In general, $msw(\rv(\overline{X}),Y)$ denotes a probabilistic switch where $\overline{X}$ is a vector of discrete variables, and $Y$ is either a discrete or continuous variable.

\textbf{Constraints, \emph{Cnstrnt}}\\
In addition to probabilistic switches, programs also contain constraints over continuous random variables. Constraints can be either in one of the following three forms: $Y=\overline{A}'\cdot\overline{X}+B$, or $Y > v$, or $Y < v$.  Here, the first constraint denotes a linear function ($Y=\overline{A}'\cdot\overline{X}+B$) where $\overline{A}$ and $B$ are constants, and the last two constraints denote comparisons with constants (e.g., $V > 1.5$). We refer to those constraints as $Cnstrnt$ in this paper.
}

\comment{
\subsection{SYMBOLIC DERIVATION}

We use the notion of \emph{symbolic derivation} for derivations involving random variables.
Note that, a derivation step $G_{1} \rightarrow G_{2}$ (where $G_{1}$ and $G_{2}$ are logic goals) can be any of the following three types.

\textbf{Derivation Type 1}
\begin{figure}[H]
\begin{displaymath}
\tiny{
  \xymatrix @=10pt{
    q_{1}(\overline{X_{1}}), q_{2}(\overline{X_{2}}),\ldots, q_{N}(\overline{X_{N}}). \ar[d]_\theta    \\
    (r_{11}(\overline{Y_{11}}),\ldots,r_{1M}(\overline{Y_{M}}), q_{2}(\overline{X_{2}}),\ldots, q_{N}(\overline{X_{N}}))_{\theta}.
  }
}
\end{displaymath}
\caption{Derivation type 1}
\label{der1}
\end{figure}

Here $q_{1}$ unifies with $q$ such that there is a predicate definition $q :- r_{11},\ldots,r_{1M}$, and $\theta = MGU(q_{1},q)$. MGU refers to the most general unifier ~\cite{aibook}.

For this type of derivation, no new variable type information is obtained. So, $D_{G_{1}} = D_{G_{2}}$.

\textbf{Derivation Type 2}\\
The second type of derivation involves a probabilistic switch (msw), and it particularly signifies the meaning of symbolic derivations.  Here, goal $G_{1}$ has a $msw$ definition, and instead of choosing multiple paths for each value of $Y$, the derivation has a single path leading to goal $G_{2}$. It means after calling $msw(rv(\overline{X}), Y)$, variable $Y$ symbolically contains all the values along with their distributions (pmf/pdf) from this program point.

\begin{figure}[H]
\begin{displaymath}
\tiny{
  \xymatrix @=10pt{
    msw(rv(\overline{X}),Y), q_{2}(\overline{X_{2}}),\ldots, q_{N}(\overline{X_{N}}). \ar[d]   \\
    q_{2}(\overline{X_{2}}),\ldots, q_{N}(\overline{X_{N}}).
  }
}
\end{displaymath}
\caption{Derivation type 2}
\label{der2}
\end{figure}

Note that, this derivation introduces variable type information because of the $msw$ definition. So, $D_{G_{1}} = D_{G_{2}} \cup \{\overline{X}:d, Y:d/c\}$, where according to the $msw$ definition $Y$ is either discrete or continuous and $\overline{X}$ is discrete.

\textbf{Derivation Type 3}\\
The third type of derivation involves a constraint (i.e., a linear function or comparison).

\begin{figure}[H]
\begin{displaymath}
\tiny{
  \xymatrix @=10pt{
    Cnstrnt,  q_{2}(\overline{X_{2}}),\ldots, q_{N}(\overline{X_{N}}). \ar[d]   \\
    q_{2}(\overline{X_{2}}),\ldots, q_{N}(\overline{X_{N}}).
  }
}
\end{displaymath}
\caption{Derivation type 3}
\label{der3}
\end{figure}

Since constraints are imposed only on continuous variables, this derivation also introduces variable type information. So, $D_{G_{1}} = D_{G_{2}} \cup \{Y_{\in Cnstrnt}:c\}$

}


\section{Illustrative Example}
\label{sec:kalman}

\comment{
\begin{figure}
\small{
\begin{verbatim}
kf(N, T) :-
    msw(init, S),
    kf_part(0, N, S, T).

kf_part(I, N, S, T) :-
    I < N, NextI is I+1,
    trans(S, I, NextS),
    emit(NextS, NextI, V),
    obs(NextI, V),
    kf_part(NextI, N, NextS, T).
kf_part(I, N, S, T) :- I=N, T=S.

trans(S, I, NextS) :-
    msw(trans_err, I, E),
    NextS = S + E.

emit(NextS, NextI, V) :-
    msw(obs_err, NextI, X),
    V = NextS + X.
\end{verbatim}
  \caption{Logic program for Kalman Filter.}
  \label{kalman-filter}
}
\end{figure}
}

\comment{
\begin{figure}
\begin{minipage}[t]{3.0in}
\begin{center}
    \begin{displaymath}
    \tiny{
    \xymatrix @=8pt {
        G_{1}: \txt{ kf(1, T) } \ar[d] \\
        G_{2}: \txt{ msw(init, S), kf\_part(0, 1, S, T) } \ar[d]  \\
        G_{3}: \txt{ kf\_part(0, 1, S, T) } \ar[d] \\
        G_{4}: \txt{ 0 $<$ 1, NextI is 0+1, trans(S, NextS), \\ emit(NextS, V), obs(NextI, V), \\kf\_part(NextI, 1, NextS, T) } \ar[d] \\
        G_{5}: \txt{ NextI is 0+1, trans(S, NextS), \\ emit(NextS, V), obs(NextI, V), \\ kf\_part(NextI, 1, NextS, T) } \ar[d]\\
        G_{6}: \txt{ trans(S, NextS), emit(NextS, V), \\ obs(1, V),  kf\_part(1, 1, NextS, T) } \ar[d] \\
        G_{7}: \txt{ msw(trans\_err, E), NextS = S + E, \\ emit(NextS, V), obs(1, V),  \\ kf\_part(1, 1, NextS, T) } \ar[d] \\
        G_{8}: \txt{ NextS = S + E, emit(NextS, V), \\ obs(1, V), kf\_part(1, 1, NextS, T) } \ar[d] \\
        G_{9}: \txt{ emit(NextS, V), obs(1, V), \\ kf\_part(1, 1, NextS, T) } \ar[d] \\
        G_{10}: \txt{ msw(obs\_err, X), V = NextS + X, \\ obs(1, V),  kf\_part(1, 1, NextS, T) } \ar[d]\\
        G_{11}: \txt{ V = NextS + X, obs(1, V), \\ kf\_part(1, 1, NextS, T) } \ar[d] \\
        G_{12}: \txt{ obs(1, V), kf\_part(1, 1, NextS, T) } \ar[d] \\
        G_{13}: \txt{ kf\_part(1, 1, NextS, T) } \ar[d] \\
        G_{14}: \txt{ 1 = 1, T = NextS }  \ar[d] \\
        G_{15}: \txt{ T = NextS }  \ar[d] \\
        \diamond
    }
    }
\end{displaymath}
\end{center}
\end{minipage}
  \caption{Symbolic Derivation of Kalman filter}
  \label{fig:kf-derivation}
 \end{figure}
}

\begin{wrapfigure}[19]{r}{0.47\textwidth}
  \begin{center}
    \begin{minipage}[t]{2in}
\small
\renewcommand{\baselinestretch}{0.9}
\begin{verbatim}
kf(N, T) :- 
  msw(init, S),
  kf_part(0, N, S, T).

kf_part(I, N, S, T) :- 
  I < N, NextI is I+1, 
  trans(S, I, NextS), 
  emit(NextS, NextI, V), 
  obs(NextI, V), 
  kf_part(NextI, N, NextS, T).

kf_part(I, N, S, T) :- 
  I=N, T=S.

trans(S, I, NextS) :- 
  msw(trans_err, I, E), 
  NextS = S + E.

emit(NextS, I, V) :- 
  msw(obs_err, I, X),
  V = NextS + X.
\end{verbatim}
    \end{minipage}
 \end{center}
  \caption{Logic program for Kalman Filter}
  \label{kalman-filter}
\vspace*{-5pt}
\end{wrapfigure}
In this section, we model Kalman filters~\cite{aibook} using logic
programs. The model describes a random walk of a single continuous
state variable $S_{t}$ with noisy observation $V_{t}$. The initial
state distribution is assumed to be Gaussian with mean $\mu_{0}$, and
variance $\sigma_{0}^{2}$.
 The transition and sensor models are 
Gaussian noises with zero means and constant variances $\sigma_{s}^{2}$, $\sigma_{v}^{2}$ respectively.

Fig.~\ref{kalman-filter} shows a logic program for Kalman filter, and
Fig.~\ref{fig:kf-derivation} shows the derivation for a query
$kf(1,T)$. Note the similarity between this and \texttt{hmm} program (Fig.~\ref{fig:hmm-prism}): only trans/emit definitions are different. We label the $i^{th}$ derivation step by $G_{i}$ which is used in the next subsection to refer to appropriate derivation step. Here, our goal is to compute filtered distribution of state $T$.

\paragraph{Success Function Computation:}

Fig.~\ref{fig:kf-derivation} shows the bottom-up success function computation.
Note that $\psi_{G_{12}}$ is same as $\psi_{G_{13}}$ except that
$obs(1, V)$ binds $V$ to an observation $v_{1}$.
Final step involves marginalization w.r.t. $S$,
\begin{align*}
\psi_{G_{1}} &= \textstyle \marginalize(\psi_{G_2}, S)\\
 &= \textstyle \mathcal{N}_{v_{1} - T}(0,\sigma_{v}^{2}).\mathcal{N}_{T}(\mu_{0},\sigma_{0}^{2}+\sigma_{s}^{2}).
 \textrm{       (using Equation~\ref{intgr-eq})}\\
 &= \textstyle \mathcal{N}_{T}(v_{1},\sigma_{v}^{2}).\mathcal{N}_{T}(\mu_{0},\sigma_{0}^{2}+\sigma_{s}^{2}).
 \textrm{          (constant shifting)}\\
 &= \textstyle \mathcal{N}_{T}\left(\frac{(\sigma_{0}^{2}+\sigma_{s}^{2})*v_{1} + \sigma_{v}^{2}*\mu_{0}}{\sigma_{0}^{2}+\sigma_{s}^{2}+\sigma_{v}^{2}}, \frac{(\sigma_{0}^{2}+\sigma_{s}^{2})*\sigma_{v}^{2}}{\sigma_{0}^{2}+\sigma_{s}^{2}+\sigma_{v}^{2}}\right).\\
 &\textrm{ (product of two Gaussian PDFs is another PDF)}
\end{align*}
which is the filtered distribution of state $T$ after seeing one observation, which is equal to the filtered distribution presented in~\cite{aibook}.

\begin{figure}
\setlength{\unitlength}{0.5cm}
\tiny{
\begin{picture}(25,28)
\put(0.0,27.5){ $G_{1}: \txt{ kf(1, T) } $} 
\put(0.0,25.5){ $ G_{2}: \txt{ msw(init, S), kf\_part(0, 1, S, T) } $}
\put(0.0, 23.5){ $ G_{3}: \txt{ kf\_part(0, 1, S, T) } $}
\put(0.0,21.5){ $ G_{4}: \txt{ 0 $<$ 1, NextI is 0+1, trans(S, NextS), \\   emit(NextS, V), obs(NextI, V), \\kf\_part(NextI, 1, NextS, T) } $}
\put(0.0,19.5){ $ G_{5}: \txt{ NextI is 0+1, trans(S, NextS), \\ emit(NextS, V), obs(NextI, V), \\ kf\_part(NextI, 1, NextS, T) }$}
\put(0.0,17.5){ $ G_{6}: \txt{ trans(S, NextS), emit(NextS, V), \\ obs(1, V),  kf\_part(1, 1, NextS, T) } $}
\put(0.0,15.5){ $ G_{7}: \txt{ msw(trans\_err, E), NextS = S + E, \\ emit(NextS, V), obs(1, V),  \\ kf\_part(1, 1, NextS, T) }  $}
\put(0.0,13.5){ $ G_{8}: \txt{ NextS = S + E, emit(NextS, V), \\ obs(1, V), kf\_part(1, 1, NextS, T) } $}
\put(0.0,11.5){ $ G_{9}: \txt{ emit(NextS, V), obs(1, V), \\ kf\_part(1, 1, NextS, T) } $}
\put(0.0,9.5){ $ G_{10}: \txt{ msw(obs\_err, X), V = NextS + X, \\ obs(1, V),  kf\_part(1, 1, NextS, T) } $}
\put(0.0,7.5){ $ G_{11}: \txt{ V = NextS + X, obs(1, V), \\ kf\_part(1, 1, NextS, T) }  $}
\put(0.0,5.7){ $ G_{12}: \txt{ obs(1, V), kf\_part(1, 1, NextS, T) }  $}
\put(0.0,4.2){ $ G_{13}: \txt{ kf\_part(1, 1, NextS, T) }  $}
\put(0.0,2.7){ $ G_{14}: \txt{ 1 = 1, T = NextS }  $}
\put(0.0,1.2){ $ G_{15}: \txt{ T = NextS } $}
\put(2.2,0.2){ $ \diamond  $}

\put(2.5, 27){\vector(0,-1){1}}
\put(2.5, 25){\vector(0,-1){1}}
\put(2.5, 23){\vector(0,-1){0.8}}
\put(2.5, 21){\vector(0,-1){0.8}}
\put(2.5, 19){\vector(0,-1){0.8}}
\put(2.5, 17){\vector(0,-1){1}}
\put(2.5, 15){\vector(0,-1){1}}
\put(2.5, 13){\vector(0,-1){1}}
\put(2.5, 11){\vector(0,-1){1}}
\put(2.5, 9){\vector(0,-1){1}}
\put(2.5, 7){\vector(0,-1){1}}
\put(2.5, 5.3){\vector(0,-1){0.7}}
\put(2.5, 3.8){\vector(0,-1){0.7}}
\put(2.5, 2.3){\vector(0,-1){0.7}}
\put(2.5, 1){\vector(0,-1){0.6}}

\put(12,27.5){ $\psi_{G_{1}}=\mathcal{N}_{T}\left(\frac{(\sigma_{0}^{2}+\sigma_{s}^{2})*v_{1} + \sigma_{v}^{2}*\mu_{0}}{\sigma_{0}^{2}+\sigma_{s}^{2}+\sigma_{v}^{2}}, \frac{(\sigma_{0}^{2}+\sigma_{s}^{2})*\sigma_{v}^{2}}{\sigma_{0}^{2}+\sigma_{s}^{2}+\sigma_{v}^{2}}\right) $}
\put(12,25.5){ $ \psi_{G_{2}} = \psi_{msw(init)} * \psi_{G_{3}} $}
\put(13.2,24.9){ $ = \mathcal{N}_{v_{1} - T}(0,\sigma_{v}^{2}).\mathcal{N}_{T - S}(0,\sigma_{s}^{2}).\mathcal{N}_{S}(\mu_{0},\sigma_{0}^{2})$}
\put(12,23.5){ $ \psi_{G_{3}} = \marginalize(\psi_{G_4}, NextS) $}
\put(13.2,22.9){ $ = \mathcal{N}_{v_{1} - T}(0,\sigma_{v}^{2}).\mathcal{N}_{T - S}(0,\sigma_{s}^{2})$}
\put(12,21.5){ $ \psi_{G_{4}} = \cpdf{\mathcal{N}_{v_{1} - NextS}(0,\sigma_{v}^{2}).\mathcal{N}_{NextS - S}(0,\sigma_{s}^{2})}{T = NextS} $}
\put(12,19.5){ $ \psi_{G_{5}} = \cpdf{\mathcal{N}_{v_{1} - NextS}(0,\sigma_{v}^{2}).\mathcal{N}_{NextS - S}(0,\sigma_{s}^{2})}{T = NextS} $}
\put(12,17.5){ $ \psi_{G_{6}} = \marginalize(\psi_{G_7}, E) $}
\put(13.2,16.9){ $ = \cpdf{\mathcal{N}_{v_{1} - NextS}(0,\sigma_{v}^{2}).\mathcal{N}_{NextS - S}(0,\sigma_{s}^{2})}{T = NextS} $}
\put(12,15.5){ $ \psi_{G_{7}} = \psi_{msw(trans\_err)} * \psi_{G_{8}} $}
\put(13.2,14.9){ $ =  \langle \mathcal{N}_{v_{1} - NextS}(0,\sigma_{v}^{2}).\mathcal{N}_{E}(0,\sigma_{s}^{2}),$}  
\put(13.5,14.2){ \textcolor{blue}{$T =NextS \wedge NextS = S + E$}$\rangle$ }
\put(12,13.5){ $ \psi_{G_{8}} = \psi_{NextS = S + E} * \psi_{G_{9}} $}
\put(13.2,12.9){ $ = \cpdf{\mathcal{N}_{v_{1} - NextS}(0,\sigma_{v}^{2})}{T = NextS \wedge NextS = S + E} $}
\put(12,11.5){ $ \psi_{G_{9}} = \marginalize(\psi_{G_{10}}, X) $}
\put(13.2,10.9){ $ = \cpdf{\mathcal{N}_{v_{1} - NextS}(0,\sigma_{v}^{2})}{T = NextS} $}
\put(12,9.5){ $ \psi_{G_{10}} =\psi_{msw(obs\_err)}*\psi_{G_{11}} $}
\put(13.4,8.9){ $ = \cpdf{\mathcal{N}_{X}(0,\sigma_{v}^{2})}{T = NextS \wedge v_{1} = NextS + X} $}
\put(12,7.5){ $ \psi_{G_{11}} = \psi_{v_{1} = NextS + X} * \psi_{G_{12}}  $}
\put(13.4,6.9){ $= \cpdf{1}{T = NextS \wedge v_{1} = NextS + X}  $}
\put(12,5.7){ $ \psi_{G_{12}} = \cpdf{1}{T = NextS} $}
\put(12,4.2){ $ \psi_{G_{13}} = \cpdf{1}{T = NextS} $}
\put(12,2.7){ $ \psi_{G_{14}} = \cpdf{1}{T = NextS} $}
\put(12,1.2){ $ \psi_{G_{15}} = \cpdf{1}{T = NextS} $}

\comment{
\put(16.5, 25.8){\vector(0,1){0.8}}
\put(16.5, 23.8){\vector(0,1){0.8}}
\put(16.5, 21.8){\vector(0,1){0.8}}
\put(16.5, 19.8){\vector(0,1){0.8}}
\put(16.5, 17.8){\vector(0,1){0.8}}
\put(16.5, 15.8){\vector(0,1){1}}
\put(16.5, 13.8){\vector(0,1){1}}
\put(16.5, 11.8){\vector(0,1){1}}
\put(16.5, 9.8){\vector(0,1){1}}
\put(16.5, 7.8){\vector(0,1){1}}
\put(16.5, 6.0){\vector(0,1){0.8}}
\put(16.5, 4.5){\vector(0,1){0.8}}
\put(16.5, 3){\vector(0,1){0.8}}
\put(16.5, 1.5){\vector(0,1){0.8}}
}
\end{picture}
}
\caption{Symbolic derivation and success functions for \texttt{kf(1,T)}}
\label{fig:kf-derivation}
\vspace{-3pt}
\end{figure}
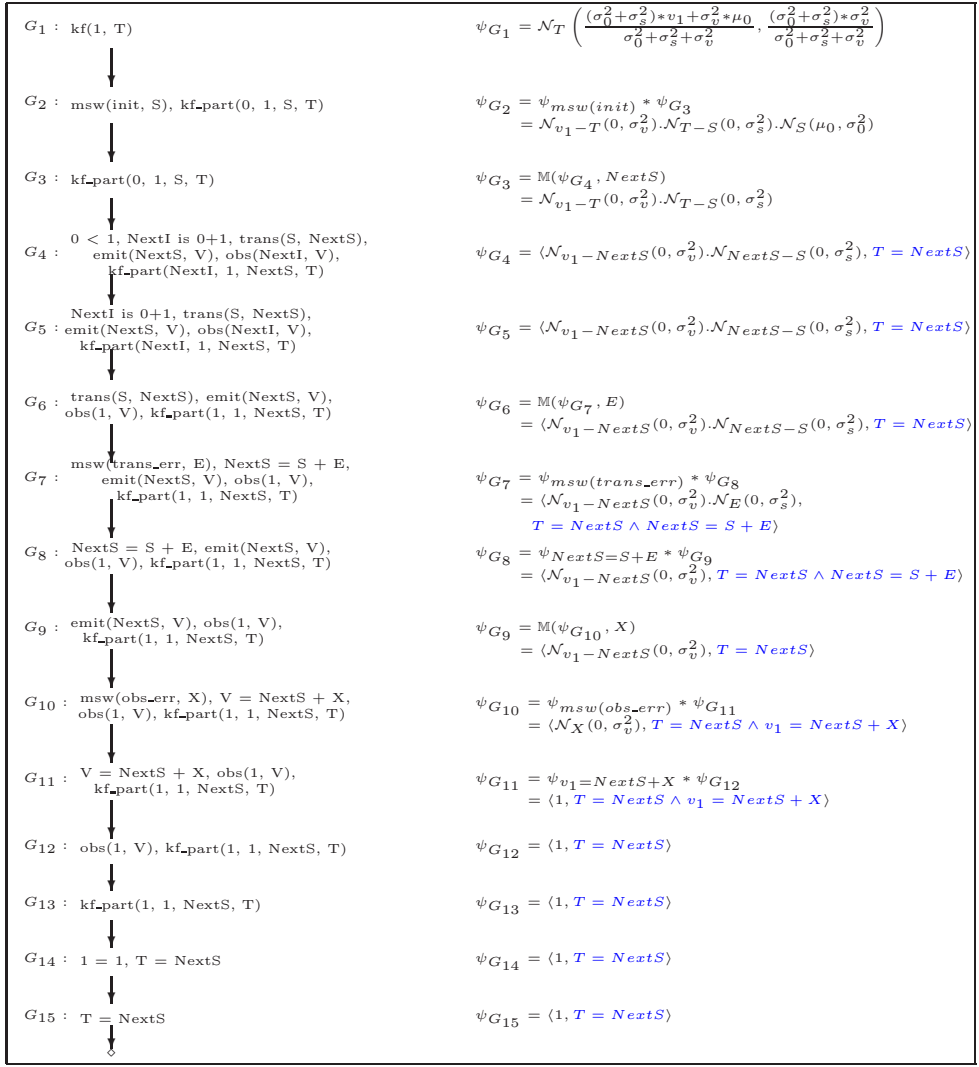

\comment{
Using our definition of success functions, the success function of the
leaf goal in Fig.~\ref{fig:kf-derivation} ($G_{15}$) is $\psi_{G_{15}} = \cpdf{1}{T = NextS}$

$\psi_{G_{13}}$ and $\psi_{G_{14}}$ are same as $\psi_{G_{15}}$.

$\psi_{G_{12}}$ is same as $\psi_{G_{13}}$ except that
$obs(1, V)$ binds $V$ to an observation $v_{1}$. Thus, $\psi_{G_{11}}$ is $\psi_{v_{1} = NextS + X} * \psi_{G_{12}}$ which yields 
\begin{equation*}
\psi_{G_{11}} = \cpdf{1}{T = NextS \wedge v_{1} = NextS + X}.
\end{equation*}
Now $\psi_{G_{10}}$ is $\psi_{msw(obs\_err)}*\psi_{G_{11}}$ which gives
\begin{equation*}
\psi_{G_{10}} = \cpdf{\mathcal{N}_{X}(0,\sigma_{v}^{2})}{T = NextS \wedge v_{1} = NextS + X}.
\end{equation*}
Marginalizing $\psi_{G_{10}}$ over $X$ yields $\psi_{G_{9}}$. Thus,
\begin{align*}
\psi_{G_{9}} = \marginalize(\psi_{G_{10}}, X) 
= \cpdf{\mathcal{N}_{v_{1} - NextS}(0,\sigma_{v}^{2})}{T = NextS}.
\end{align*}
Similarly,
\begin{align*}
\psi_{G_{8}} &= \psi_{NextS = S + E} * \psi_{G_{9}}
 = \cpdf{\mathcal{N}_{v_{1} - NextS}(0,\sigma_{v}^{2})}{T = NextS \wedge NextS = S + E}.\\
\psi_{G_{7}} &= \psi_{msw(trans\_err)} * \psi_{G_{8}}
 =  \cpdf{\mathcal{N}_{v_{1} - NextS}(0,\sigma_{v}^{2}).\mathcal{N}_{E}(0,\sigma_{s}^{2})}  
  {T =NextS \wedge NextS = S + E}.\\
\psi_{G_{6}} &= \marginalize(\psi_{G_7}, E) 
= \cpdf{\mathcal{N}_{v_{1} - NextS}(0,\sigma_{v}^{2}).\mathcal{N}_{NextS - S}(0,\sigma_{s}^{2})}{T = NextS}.
\end{align*}
$\psi_{G_{4}}$ and $\psi_{G_{5}}$ are same as $\psi_{G_{6}}$.
Next,
\begin{align*}
\psi_{G_{3}} &= \marginalize(\psi_{G_4}, NextS)
= \mathcal{N}_{v_{1} - T}(0,\sigma_{v}^{2}).\mathcal{N}_{T - S}(0,\sigma_{s}^{2}).\\
\psi_{G_{2}} &= \psi_{msw(init)} * \psi_{G_{3}}
 = \mathcal{N}_{v_{1} - T}(0,\sigma_{v}^{2}).\mathcal{N}_{T - S}(0,\sigma_{s}^{2}).\mathcal{N}_{S}(\mu_{0},\sigma_{0}^{2}).
\end{align*}
Finally,
\begin{align*}
\psi_{G_{1}} &= \marginalize(\psi_{G_2}, S)\\ 
 &= \mathcal{N}_{v_{1} - T}(0,\sigma_{v}^{2}).\mathcal{N}_{T}(\mu_{0},\sigma_{0}^{2}+\sigma_{s}^{2}).
 \textrm{       (using Equation~\ref{intgr-eq})}\\
 &= \mathcal{N}_{T}(v_{1},\sigma_{v}^{2}).\mathcal{N}_{T}(\mu_{0},\sigma_{0}^{2}+\sigma_{s}^{2}).
 \textrm{          (constant shifting)}\\
 &= \mathcal{N}_{T}\left(\frac{(\sigma_{0}^{2}+\sigma_{s}^{2})*v_{1} + \sigma_{v}^{2}*\mu_{0}}{\sigma_{0}^{2}+\sigma_{s}^{2}+\sigma_{v}^{2}}, \frac{(\sigma_{0}^{2}+\sigma_{s}^{2})*\sigma_{v}^{2}}{\sigma_{0}^{2}+\sigma_{s}^{2}+\sigma_{v}^{2}}\right).\\
 &\textrm{ (product of two Gaussian PDFs is another PDF)}\\
\end{align*}
which is the filtered distribution of state $T$ after seeing one observation, which is equal to the filtered distribution presented in~\cite{aibook}.
}


\section{Discussion and Concluding Remarks}
\label{sec:extn}

ProbLog and PITA~\cite{RiguzziSwift10b}, an implementation of LPAD, lift PRISM's mutual exclusion and independence restrictions by using a BDD-based representation of explanations. 
\comment{ Their inference technique first materializes the set of explanations for each query, and represents this set as a BDD, where each node in the BDD is a (discrete) random variable.  Distinct paths in the BDD are mutually exclusive and variables in a single path are all independent. Probabilities of query answers are computed with a simple dynamic programming algorithm using this BDD representation.}
The technical development in this paper is based on PRISM and imposes PRISM's restrictions.  However, we can remove these restrictions by using the following approach.  In the first step, we materialize the set of symbolic derivations.   In the second step, we can \emph{factor} the derivations into a form analogous to BDDs such that random variables each path of the factored representation are independent, and distinct paths in the representation are mutually exclusive.  For instance, consider two non-exclusive branches in a symbolic derivation tree, one of which has \texttt{msw(r, X)} and the other that has \texttt{msw(s,Y)}.  This will be factored such that one of the two, say \texttt{msw(r,X')} is done in common, with two branches: $X=X'$ and $X\not=X'$.  The branch containing subgoal \texttt{msw(s,Y)} is ``and-ed'' with the $X=X'$ branch, and replicated as the $X\not=X'$ branch, analogous to how BDDs are processed.  The factored representation itself can be treated as symbolic derivations augmented with dis-equality constraints (i.e. of the form $X\not= e$).  Note that the success function of an equality constraint $C$ is $\cpdf{1}{C}$.  The success function of a dis-equality constraint $X\not= e$ is $\cpdf{1}{\id{true}} - \cpdf{1}{X=e}$, which is representable by extending our language of success functions to permit non-negative constants.  The definitions of join and marginalize operations work with no change over the extended success functions, and the closure properties (Prop.~\ref{opclosure}) holds as well.  Hence, success functions can be readily computed over the factored representation.  A detailed discussion of this extension appears in~\cite{AsifulIslam2012}.

Note that the success function of a goal represents the likelihood of a successful derivation for each instance of a goal.  Hence the probability measure computed by the success function is what PRISM calls \emph{inside} probability.  Analogously, we can define a function that represents the likelihood that a goal $G'$ will be encountered in a symbolic derivation starting at goal $G$. This ``call'' function will represent the \emph{outside} probability of PRISM.  Alternatively, we can use the Magic Sets transformation~\cite{magicset} to compute call functions of a program in terms of success functions of a transformed program.  The ability to compute inside and outside probabilities can be used to infer smoothed distributions for temporal models.

For simplicity, in this paper we focused  only on univariate
Gaussians. However, the techniques can be easily extended
to support multivariate Gaussian distributions, by extending the
integration function (Defn. ~\ref{intgop}), and \texttt{set\_sw} directives.
We can also readily extend them to support Gamma distributions.  More generally,
the PDF functions can be generalized to contain
Gaussian or Gamma density functions, such that
variables are not shared between Gaussian and Gamma
density functions.  Again, the only change is
to extend the integration function to handle PDFs of Gamma distribution.

The concept of symbolic derivations and success functions can be applied to parameter learning as well.  We have developed an EM-based learning algorithm which permits us to learn the distribution parameters of extended PRISM programs with discrete as well as Gaussian random variables~\cite{contdist-learning}. Similar to inference, our learning algorithm uses the \emph{symbolic derivation} procedure to compute Expected Sufficient Statistics (ESS). The E-step of the learning algorithm involves computation of the ESSs of the random variables and the M-step computes the MLE of the distribution parameters given the ESS and success probabilities.  Analogous to the inference algorithm presented in this paper, our learning algorithm specializes to PRISM's learning over programs without any continuous variables.  For mixture model, the learning algorithm does the same computation as standard EM learning algorithm~\cite{bishop}.

The symbolic inference and learning procedures enable us to reason over 
a large class of statistical models such as hybrid Bayesian networks with discrete child-discrete parent, continuous child-discrete parent (finite mixture model), and continuous child-continuous parent (Kalman filter), which was hitherto not possible in PLP frameworks.  It can also be used for  hybrid models, e.g., models that mix discrete and Gaussian distributions.  For instance, consider the mixture model example where \texttt{st(a)} is Gaussian but \texttt{st(b)} is a discrete distribution with values $1$ and $2$ with $0.5$ probability each. The density of the mixture distribution can be written as 
$
f(Z) = 0.3 \mathcal{N}_{Z}(2.0, 1.0)  + 0.35  \delta_{1.0}(Z)  + 0.35  \delta_{2.0}(Z).
$
Thus the language can be used to model problems that lie outside traditional hybrid Bayesian networks.

We implemented the extended inference algorithm presented in this
paper in the XSB logic programming system~\cite{XSB}.  The system is
available at
\url{http://www.cs.sunysb.edu/~cram/contdist}. This
proof-of-concept prototype is implemented as a meta-interpreter and
currently supports discrete and Gaussian distributions. The meaning of
various probabilistic predicates (e.g., \texttt{msw, values, set\_sw})
in the system are similar to that of PRISM system.  This
implementation illustrates how the inference algorithm specializes to
the specialized techniques that have been developed for several
popular statistical models such as HMM, FMM, Hybrid Bayesian Networks
and Kalman Filters. Integration of the inference
algorithm in XSB and its performance evaluation are topics of future work.
\paragraph{Acknowledgments.}
We thank the reviewers for valuable comments.
This research  was supported in part by 
NSF Grants CCF-\mbox{1018459}, 
CCF-\mbox{0831298}, 
and ONR Grant N00014-07-1-0928. 


\section{Appendix}
\label{sec:ppdfint}

This section presents proof of Proposition~\ref{integration}.

\begin{Prop}
\label{PPDF-integration}
Integrated form of a PPDF function with respect to a variable $V$ is a PPDF function, i.e., 
\begin{align*}
\int^{\infty}_{-\infty} \prod_{k=1}^{m}
\mathcal{N}_{(\overline{a_{k}} \cdot
  \overline{X_{k}}+b_{k})}(\mu_{k},\sigma^{2}_{k}) dV = \alpha \prod_{l=1}^{m'}
\mathcal{N}_{(\overline{a'_{l}} \cdot
  \overline{X'_{l}})+b'_{l}}(\mu'_{l},\sigma'^{2}_{l})
\end{align*}
where $V \in \overline{X_{k}}$ and $V \notin \overline{X'_{l}}$.
\end{Prop}

\noindent(Proof)\\
The above proposition states that integrated form of a product of Gaussian PDF functions with respect to a variable is a product of Gaussian PDF functions. We first prove it for a simple case involving two standard Gaussian PDF functions, and then generalize it for arbitrary number of Gaussians.

For simplicity, let us first compute the integrated-form of $\mathcal{N}_{V-X_{1}}(0, 1) . \mathcal{N}_{V-X_{2}}(0, 1)$ w.r.t. variable $V$ where $X_{1}, X_{2}$ are linear combination of variables (except $V$).
We make the following two assumptions: \\
1. The coefficient of $V$ is $1$ in both PDFs. \\
2. Both PDFs are standard normal distributions (i.e., $\mu=0$ and $\sigma^{2} = 1)$.

Let $\phi$ denote the integrated form, i.e.,
\begin{align*}
 \phi &= \int^{\infty}_{-\infty}  \mathcal{N}_{V-X_{1}}(0, 1) . \mathcal{N}_{V-X_{2}}(0, 1) dV \\
    &= \int^{\infty}_{-\infty}  \frac{1}{\sqrt{2\pi}}\exp^{-\frac{(V - X_{1})^2}{2}} . \frac{1}{\sqrt{2\pi}}\exp^{-\frac{(V - X_{2})^2}{2}} dV \\
     &= \int^{\infty}_{-\infty}  \frac{1}{2\pi}\exp^{-\frac{1}{2}[(V - X_{1})^2 + (V - X_{2})^2]} dV\\
     &= \int^{\infty}_{-\infty} \frac{1}{2\pi}\exp^{-\frac{1}{2}\eta} dV
\end{align*}
Now 
 \begin{align*}
 \eta &= (V - X_{1})^2 + (V - X_{2})^2 \\
   &= 2.V^{2} - 2.V.(X_{1} + X_{2}) + (X_{1}^{2} + X_{2}^{2}) \\
   &= 2 [(V - \frac{X_{1} + X_{2}}{2})^{2} + (\frac{X_{1}^{2} + X_{2}^{2}}{2}) - (\frac{X_{1} + X_{2}}{2}) ^ {2}] \\
   &= 2 [(V - \frac{X_{1} + X_{2}}{2})^{2} + g] 
  \end{align*}
where 
 \begin{align*}
 g = (\frac{X_{1}^{2} + X_{2}^{2}}{2}) - (\frac{X_{1} + X_{2}}{2}) ^ {2} 
  = \frac{1}{4}(X_{1} - X_{2})^{2}
  \end{align*}
Thus the integrated form can be expressed as 
\begin{align*}
 \phi  &= \int^{\infty}_{-\infty}  \frac{1}{2\pi}\exp^{-\frac{1}{2}. 2 [(V - \frac{X_{1} + X_{2}}{2})^{2} + g] } dV \\
  &= \int^{\infty}_{-\infty}  \frac{1}{2\pi}\exp^{-\frac{1}{2}. 2.(V - \frac{X_{1} + X_{2}}{2})^{2}}. \exp^{-\frac{1}{2}. 2.g}  dV \\
   &=  \frac{1}{2\sqrt{\pi}} exp^{- g}  \int^{\infty}_{-\infty}  \frac{1}{\sqrt{2.\pi.\frac{1}{2}}}\exp^{-\frac{(V - \frac{X_{1} + X_{2}}{2})^{2}}{2.\frac{1}{2}}}  dV \\
   &=  \frac{1}{2\sqrt{\pi}} exp^{- g} \texttt{   (as integration over the whole area is 1)} \\
   &=  \frac{1}{\sqrt{2\pi.2}} exp^{- \frac{1}{4}(X_{1} - X_{2})^{2}} \\
   &=  \mathcal{N}_{X_{1} - X_{2}}(0, 2) 
\end{align*}
 Thus integrated form of a PPDF function is another PPDF function. 
Notice that the integrated form is a constant when $X1 = X2$.

\paragraph{Generalization for arbitrary number of PDFs.} 
Note that for any arbitrary number of PDFs in a PPDF function, 
$\eta = \sum (V - X_i)^{2}$ can be always written as $k[(V - \beta)^2 + g_{n}]$, where 
\begin{align*}
 g_{n} &= \frac{1}{n}\sum_{i=1}^{n} X_{i}^{2}  - \frac{1}{n^2}(\sum_{i=1}^{n} X_{i})^{2}  
  \end{align*}
For any arbitrary number of PDFs, we will prove the property on $g_{n}$. 
In other words, we will show that $g_{n}$ can be expressed as
 \begin{equation}
 \label{gn}
 g_{n} = \frac{1}{n}\sum_{i=1}^{n} X_{i}^{2}  - \frac{1}{n^2}(\sum_{i=1}^{n} X_{i})^{2} 
 = \frac{1}{n^{2}}\sum_{i \ne j, i < j} (X_{i} - X_{j})^{2}  
  \end{equation}
which means integrated form of $n$ PDFs, 
\begin{align*}
\phi_{n} =  \int^{\infty}_{-\infty} \prod_{i=1}^{n} \mathcal{N}_{V-X_{i}}(0, 1) dV
\end{align*}
 can be expressed as 
 \begin{align*}
 \phi_{n} = \alpha \exp^{-g_{n}} = \alpha  \prod_{i \ne j, i < j} \mathcal{N}_{X_{i} - X_{j}} 
  \end{align*}

\begin{Pro}
\label{fn}
 Let $f_{n}= \sum_{i=1}^{n} X_{i}$. Then, 
 $f_{n}^{2}= \sum_{i=1}^{n} X_{i}^{2} +  \sum_{i \ne j} X_{i} X_{j}$.
\end{Pro} 
\begin{proof}
 We prove the proposition using induction.
 Let us assume that the above equation holds for  $n$ variables.
 Now for $(n+1)^{th}$ variable $X_{n+1}$,
 \begin{align*}
 f_{n+1}^{2} &= (\sum_{i=1}^{n+1} X_{i})^{2} \\
  &= ((\sum_{i=1}^{n} X_{i}) + X_{n+1})^{2} \\
 &= (\sum_{i=1}^{n} X_{i})^{2} + X_{n+1}^{2} + 2(X_{1} + ... + X_{n})X_{n+1} \\ 
 &= \sum_{i=1}^{n} X_{i}^{2} + \sum_{i \ne j, i=1}^{n} X_{i} X_{j} +  X_{n+1}^{2} + 2(X_{1} + ... + X_{n})X_{n+1}  \\
  &\texttt{   (using induction hypothesis)}\\
 &= \sum_{i=1}^{n+1} X_{i}^{2} + \sum_{i \ne j} X_{i} X_{j}
  \end{align*} 
\end{proof}

Now going back to proving equation~\ref{gn},
we first show that $g_{n}$ can be written in the following form
 \begin{align*}
 g_{n} = \frac{1}{n}\sum_{i=1}^{n} X_{i}^{2}  - \frac{1}{n^2}(\sum_{i=1}^{n} X_{i})^{2} 
 = \frac{1}{n^{2}} [ (n-1) \sum_{i=1}^{n} X_{i}^{2}  - \sum_{i \ne j} X_{i} X_{j} ] 
  \end{align*}
 
 The above equation can be proved by induction.
It is easy to see that for $n=2$ the equation holds, as 
$g_{2} = \frac{1}{2}\sum_{i=1}^{2} X_{i}^{2}  - \frac{1}{4}(\sum_{i=1}^{2} X_{i})^{2} 
 = \frac{1}{4} [ X_{1}^{2} + X_{2}^{2} - X_{1} X_{2} - X_{2} X_{1} ] $.
Now 
 \begin{align*}
 g_{n+1} &= \frac{1}{(n+1)}\sum_{i=1}^{n+1} X_{i}^{2}  -  \frac{1}{(n+1)^{2}} f_{n+1}^{2}\\
  &= \frac{1}{(n+1)^{2}} [ (n+1) \sum_{i=1}^{n+1} X_{i}^{2} -  f_{n+1}^{2}] \\
 &= \frac{1}{(n+1)^{2}} [ (n+1) \sum_{i=1}^{n+1} X_{i}^{2} - \sum_{i=1}^{n+1} X_{i}^{2} - \sum_{i \ne j} X_{i} X_{j} ] \\
 &\texttt{   (using Proposition~\ref{fn})}\\
 &= \frac{1}{(n+1)^{2}} [ n \sum_{i=1}^{n+1} X_{i}^{2} - \sum_{i \ne j} X_{i} X_{j} ] 
  \end{align*} 

Thus $g_{n} = \frac{1}{n^{2}} [ (n-1) \sum_{i=1}^{n} X_{i}^{2}  - \sum_{i \ne j} X_{i} X_{j} ] $.
  
Finally, we will prove that
 \begin{align*}
 g_{n}  = \frac{1}{n^{2}} [ (n-1) \sum_{i=1}^{n} X_{i}^{2}  - \sum_{i \ne j} X_{i} X_{j} ] 
  = \frac{1}{n^{2}}\sum_{i \ne j, i < j} (X_{i} - X_{j})^{2} 
  \end{align*}

\begin{Pro}
Let $h_{n} = \sum_{i \ne j, i < j} (X_{i} - X_{j})^{2}$. Then
$h_{n} =  (n-1) \sum_{i=1}^{n} X_{i}^{2}  - \sum_{i \ne j} X_{i} X_{j}$.
\end{Pro}
\begin{proof}
We use induction to prove the above proposition.
Let $h_{n}$ holds for $n$ variables. Then for  $(n+1)^{th}$ variable,
 \begin{align*}
h_{n+1} &= \sum_{i \ne j, i < j} (X_{i} - X_{j})^{2} \\
&= h_{n} +   \sum_{i=1}^{n} (X_{i} - X_{n+1})^{2}   \\
 &= (n-1) \sum_{i=1}^{n} X_{i}^{2}  - \sum_{i \ne j, i=1}^{n} X_{i} X_{j} + \sum_{i=1}^{n} X_{i}^{2} + nX_{n+1}^{2} 
  - 2\sum_{i=1}^{n} X_{i}X_{n+1}\\
&= n \sum_{i=1}^{n+1} X_{i}^{2}  - \sum_{i \ne j, i=1}^{n+1} X_{i} X_{j} 
  \end{align*}
\end{proof}

Thus $g_{n} = \frac{1}{n^{2}} h_{n} = \frac{1}{n^{2}}\sum_{i \ne j, i < j} (X_{i} - X_{j})^{2}$.
Thus $\phi_{n}$ can be expressed as 
 \begin{align*}
 \phi_{n} = \alpha \exp^{-g_{n}} = \alpha  \prod_{i \ne j, i < j} \mathcal{N}_{X_{i} - X_{j}} 
  \end{align*}

\paragraph{Integrated-form with arbitrary constants:}
For any arbitrary mean, variance and coefficients of $V$, 
\begin{align*}
 \phi_{2} &= \int^{\infty}_{-\infty} \mathcal{N}_{a_{1} V - X_{1}} (\mu_{1}, \sigma_{1}^{2}) . \mathcal{N}_{a_{2} V - X_{2}}(\mu_{2}, \sigma_{2}^{2}) dV\\
     &= \mathcal{N}_{a_{2} X_{1} - a_{1} X_{2}} (a_{1}\mu_{2} - a_{2} \mu_{1}, a_{2}^{2} \sigma_{1}^{2} + a_{1}^{2} \sigma_{2}^{2})
\end{align*}

And 
\begin{align*}
 \phi_{n} &= \alpha  \prod_{i \ne j, i < j} \mathcal{N}_{a_{j}X_{i} - a_{i}X_{j}} (a_{i}\mu_{j} - a_{j} \mu_{i}, \sigma_{ij}^{2})
\end{align*}

where 
\begin{align*}
 \sigma_{ij}^{2} &= \frac {\sum_{k=1}^{n}  a_{k}^{2} \prod_{l \ne k, l=1}^{n} \sigma_{l}^{2}} {\prod_{k=1, k \ne i,j}^{n} \sigma_{k}^{2}}
\end{align*}

Note that  the normalization constant is also adjusted  appropriately in the integrated form.

\bibliographystyle{acmtrans}
\renewcommand{\baselinestretch}{0.95}
\bibliography{biblio}

\begin{thebibliography}{}

\bibitem[\protect\citeauthoryear{Bancilhon, Maier, Sagiv, and Ullman}{Bancilhon
  et~al\mbox{.}}{1986}]{magicset}
{\sc Bancilhon, F.}, {\sc Maier, D.}, {\sc Sagiv, Y.}, {\sc and} {\sc Ullman,
  J.} 1986.
\newblock Magic sets and other strange ways to implement logic programs.
\newblock In {\em Proceedings of PODS}.

\bibitem[\protect\citeauthoryear{Bishop}{Bishop}{2006}]{bishop}
{\sc Bishop, C.} 2006.
\newblock {\em Pattern recognition and Machine Learning.}
\newblock Springer.

\bibitem[\protect\citeauthoryear{Chu, Popa, Tavakoli, Hellerstein, Levis,
  Shenker, and Stoica}{Chu et~al\mbox{.}}{2007}]{DSN}
{\sc Chu, D.}, {\sc Popa, L.}, {\sc Tavakoli, A.}, {\sc Hellerstein, J.~M.},
  {\sc Levis, P.}, {\sc Shenker, S.}, {\sc and} {\sc Stoica, I.} 2007.
\newblock The design and implementation of a declarative sensor network system.
\newblock In {\em SenSys}. 175--188.

\bibitem[\protect\citeauthoryear{Forney}{Forney}{1973}]{Viterbi}
{\sc Forney, G.} 1973.
\newblock The {Viterbi} algorithm.
\newblock In {\em Proceedings of the IEEE}. 268--278.

\bibitem[\protect\citeauthoryear{Friedman, Getoor, Koller, and
  Pfeffer}{Friedman et~al\mbox{.}}{1999}]{prm}
{\sc Friedman, N.}, {\sc Getoor, L.}, {\sc Koller, D.}, {\sc and} {\sc Pfeffer,
  A.} 1999.
\newblock Learning probabilistic relational models.
\newblock In {\em IJCAI}. 1300--1309.

\bibitem[\protect\citeauthoryear{Getoor and Taskar}{Getoor and
  Taskar}{2007}]{srlbook}
{\sc Getoor, L.} {\sc and} {\sc Taskar, B.} 2007.
\newblock {\em Introduction to Statistical Relational Learning}.
\newblock The MIT Press.

\bibitem[\protect\citeauthoryear{Goswami, Ortiz, and Das}{Goswami
  et~al\mbox{.}}{2011}]{WiGEM}
{\sc Goswami, A.}, {\sc Ortiz, L.~E.}, {\sc and} {\sc Das, S.~R.} 2011.
\newblock {WiGEM}: A learning-based approach for indoor localizatio.
\newblock In {\em SIGCOMM}.

\bibitem[\protect\citeauthoryear{Gutmann, Jaeger, and Raedt}{Gutmann
  et~al\mbox{.}}{2010}]{hproblog}
{\sc Gutmann, B.}, {\sc Jaeger, M.}, {\sc and} {\sc Raedt, L.~D.} 2010.
\newblock Extending {ProbLog} with continuous distributions.
\newblock In {\em Proceedings of ILP}.

\bibitem[\protect\citeauthoryear{Gutmann, Thon, Kimmig, Bruynooghe, and
  Raedt}{Gutmann et~al\mbox{.}}{2011}]{apprProblog}
{\sc Gutmann, B.}, {\sc Thon, I.}, {\sc Kimmig, A.}, {\sc Bruynooghe, M.}, {\sc
  and} {\sc Raedt, L.~D.} 2011.
\newblock The magic of logical inference in probabilistic programming.
\newblock {\em TPLP\/}~{\em 11,\/}~4-5, 663--680.

\bibitem[\protect\citeauthoryear{{Islam}, {Ramakrishnan}, and
  {Ramakrishnan}}{{Islam} et~al\mbox{.}}{2011}]{contdist-inference}
{\sc {Islam}, M.}, {\sc {Ramakrishnan}, C.~R.}, {\sc and} {\sc {Ramakrishnan},
  I.~V.} 2011.
\newblock {Inference in Probabilistic Logic Programs with Continuous Random
  Variables}.
\newblock {\em ArXiv e-prints\/}.
\newblock \url{http://arxiv.org/abs/1112.2681}.

\bibitem[\protect\citeauthoryear{{Islam}, {Ramakrishnan}, and
  {Ramakrishnan}}{{Islam} et~al\mbox{.}}{2012}]{contdist-learning}
{\sc {Islam}, M.}, {\sc {Ramakrishnan}, C.~R.}, {\sc and} {\sc {Ramakrishnan},
  I.~V.} 2012.
\newblock {Parameter Learning in PRISM Programs with Continuous Random
  Variables}.
\newblock {\em ArXiv e-prints\/}.
\newblock \url{http://arxiv.org/abs/1203.4287}.

\bibitem[\protect\citeauthoryear{Islam}{Islam}{2012}]{AsifulIslam2012}
{\sc Islam, M.~A.} 2012.
\newblock Inference and learning in probabilistic logic programs with
  continuous random variables.
\newblock Ph.D. thesis.
\newblock \url{http://www.cs.sunysb.edu/~cram/asiful2012.pdf}.

\bibitem[\protect\citeauthoryear{Jaffar, Maher, Marriott, and Stuckey}{Jaffar
  et~al\mbox{.}}{1998}]{JaffarCLP}
{\sc Jaffar, J.}, {\sc Maher, M.~J.}, {\sc Marriott, K.}, {\sc and} {\sc
  Stuckey, P.~J.} 1998.
\newblock The semantics of constraint logic programs.
\newblock {\em J. Log. Program.\/}~{\em 37,\/}~1-3, 1--46.

\bibitem[\protect\citeauthoryear{Kersting and Raedt}{Kersting and
  Raedt}{2000}]{blp}
{\sc Kersting, K.} {\sc and} {\sc Raedt, L.~D.} 2000.
\newblock Bayesian logic programs.
\newblock In {\em ILP Work-in-progress reports}.

\bibitem[\protect\citeauthoryear{Kersting and Raedt}{Kersting and
  Raedt}{2001}]{hblp}
{\sc Kersting, K.} {\sc and} {\sc Raedt, L.~D.} 2001.
\newblock Adaptive {Bayesian} logic programs.
\newblock In {\em ILP}.

\bibitem[\protect\citeauthoryear{Lari and Young}{Lari and
  Young}{1990}]{InsideOutside}
{\sc Lari, K.} {\sc and} {\sc Young, S.~J.} 1990.
\newblock The estimation of stochastic context-free grammars using the
  inside-outside algorithm.
\newblock {\em Computer Speech and Language\/}~{\em 4}, 3556.

\bibitem[\protect\citeauthoryear{Muggleton}{Muggleton}{1996}]{Muggleton}
{\sc Muggleton, S.} 1996.
\newblock Stochastic logic programs.
\newblock In {\em Adv. in inductive logic programming}.

\bibitem[\protect\citeauthoryear{Murphy}{Murphy}{1998}]{hbn}
{\sc Murphy, K.} 1998.
\newblock Inference and learning in hybrid {Bayesian} networks.
\newblock {\em Technical Report UCB/CSD-98-990\/}.

\bibitem[\protect\citeauthoryear{Narman, Buschle, Konig, and Johnson}{Narman
  et~al\mbox{.}}{2010}]{hprm}
{\sc Narman, P.}, {\sc Buschle, M.}, {\sc Konig, J.}, {\sc and} {\sc Johnson,
  P.} 2010.
\newblock Hybrid probabilistic relational models for system quality analysis.
\newblock In {\em Proceedings of EDOC}.

\bibitem[\protect\citeauthoryear{Poole}{Poole}{1993}]{Poole}
{\sc Poole, D.} 1993.
\newblock Probabilistic {Horn} abduction and {Bayesian} networks.
\newblock {\em Artif. Intell.\/}~{\em 64,\/}~1, 81--129.

\bibitem[\protect\citeauthoryear{Poole}{Poole}{2008}]{PooleICL}
{\sc Poole, D.} 2008.
\newblock The independent choice logic and beyond.
\newblock In {\em Probabilistic ILP}. 222--243.

\bibitem[\protect\citeauthoryear{Raedt, Kimmig, and Toivonen}{Raedt
  et~al\mbox{.}}{2007}]{deRaedt}
{\sc Raedt, L.~D.}, {\sc Kimmig, A.}, {\sc and} {\sc Toivonen, H.} 2007.
\newblock {ProbLog}: A probabilistic prolog and its application in link
  discovery.
\newblock In {\em IJCAI}. 2462--2467.

\bibitem[\protect\citeauthoryear{Richardson and Domingos}{Richardson and
  Domingos}{2006}]{mln}
{\sc Richardson, M.} {\sc and} {\sc Domingos, P.} 2006.
\newblock Markov logic networks.
\newblock {\em Machine Learning\/}.

\bibitem[\protect\citeauthoryear{Riguzzi and Swift}{Riguzzi and
  Swift}{2010}]{RiguzziSwift10b}
{\sc Riguzzi, F.} {\sc and} {\sc Swift, T.} 2010.
\newblock Tabling and answer subsumption for reasoning on logic programs with
  annotated disjunctions.
\newblock In {\em Tech. Comm. of ICLP}. 162–--171.

\bibitem[\protect\citeauthoryear{Russell and Norvig}{Russell and
  Norvig}{2003}]{aibook}
{\sc Russell, S.} {\sc and} {\sc Norvig, P.} 2003.
\newblock {\em Arficial Intelligence: A Modern Approach}.
\newblock Prentice~Hall.

\bibitem[\protect\citeauthoryear{Sato and Kameya}{Sato and
  Kameya}{1997}]{sato-kameya-prism}
{\sc Sato, T.} {\sc and} {\sc Kameya, Y.} 1997.
\newblock {PRISM:} a symbolic-statistical modeling language.
\newblock In {\em IJCAI}.

\bibitem[\protect\citeauthoryear{Sato and Kameya}{Sato and Kameya}{1999}]{sato}
{\sc Sato, T.} {\sc and} {\sc Kameya, Y.} 1999.
\newblock Parameter learning of logic programs for symbolic-statistical
  modeling.
\newblock {\em Journal of Artificial Intelligence Research\/}, 391--454.

\bibitem[\protect\citeauthoryear{Singh, Ramakrishnan, Ramakrishnan, Warren, and
  Wong}{Singh et~al\mbox{.}}{2008}]{Sensys}
{\sc Singh, A.}, {\sc Ramakrishnan, C.~R.}, {\sc Ramakrishnan, I.~V.}, {\sc
  Warren, D.}, {\sc and} {\sc Wong, J.} 2008.
\newblock A methodology for in-network evaluation of integrated
  logical-statistical models.
\newblock In {\em SenSys}. 197--210.

\bibitem[\protect\citeauthoryear{Swift, Warren, et~al\mbox{.}}{Swift
  et~al\mbox{.}}{2012}]{XSB}
{\sc Swift, T.}, {\sc Warren, D.~S.}, {\sc et~al\mbox{.}} 2012.
\newblock The {XSB} logic programming system, {V}ersion 3.3.
\newblock Tech. rep., Computer Science, SUNY, Stony Brook.
\newblock \url{http://xsb.sourceforge.net}.

\bibitem[\protect\citeauthoryear{Tamaki and Sato}{Tamaki and Sato}{1986}]{OLDT}
{\sc Tamaki, H.} {\sc and} {\sc Sato, T.} 1986.
\newblock {OLD} resolution with tabulation.
\newblock In {\em ICLP}. 84--98.

\bibitem[\protect\citeauthoryear{Vennekens, Denecker, and Bruynooghe}{Vennekens
  et~al\mbox{.}}{2009}]{CPlogic}
{\sc Vennekens, J.}, {\sc Denecker, M.}, {\sc and} {\sc Bruynooghe, M.} 2009.
\newblock {CP-logic}: A language of causal probabilistic events and its
  relation to logic programming.
\newblock {\em TPLP\/}.

\bibitem[\protect\citeauthoryear{Vennekens, Verbaeten, and
  Bruynooghe}{Vennekens et~al\mbox{.}}{2004}]{lpad}
{\sc Vennekens, J.}, {\sc Verbaeten, S.}, {\sc and} {\sc Bruynooghe, M.} 2004.
\newblock Logic programs with annotated disjunctions.
\newblock In {\em ICLP}. 431--445.

\bibitem[\protect\citeauthoryear{Wang and Domingos}{Wang and
  Domingos}{2008}]{hmln}
{\sc Wang, J.} {\sc and} {\sc Domingos, P.} 2008.
\newblock Hybrid markov logic networks.
\newblock In {\em Proceedings of AAAI}.

\end{thebibliography}

\newpage


\end{document}